\documentclass[11pt]{article}
\usepackage[hidelinks]{hyperref}
\usepackage[version = 4]{mhchem}
\usepackage[left=2.5cm,right=2.5cm,top=2.5cm,bottom=2.5cm]{geometry}

\usepackage{mlmath}
\usepackage{bm}
\usepackage{graphicx}%
\usepackage{multirow}%
\usepackage{subfigure}  
\usepackage{amsmath,amssymb,amsfonts}%
\usepackage{amsthm}%
\usepackage{mathrsfs}%
\usepackage[title]{appendix}%
\usepackage{xcolor}%
\usepackage{color}
\usepackage{textcomp}%
\usepackage{manyfoot}%
\usepackage{booktabs}%
\usepackage{algorithm}%
\usepackage{algorithmicx}%
\usepackage{algpseudocode}%
\usepackage{listings}%
\usepackage{enumitem}
\usepackage{booktabs} 
\usepackage{caption}  
\usepackage{adjustbox} 
\usepackage{array}    
\usepackage{hyperref}
\usepackage{subcaption}

\usepackage{tabularx}
\usepackage{multirow}

\newenvironment{keywords}
{\bgroup\leftskip 20pt\rightskip 20pt \small\noindent{\bf Keywords:} }%
{\par\egroup\vskip 0.25ex}

\long\def\acks#1{\vskip 0.3in\noindent{\large\bf Acknowledgments and Disclosure of Funding}\vskip 0.2in
\noindent #1}

\newtheorem{assump}{Assumption}
\newtheorem{rmk}{Remark}
\newtheorem{thm}{Theorem}
\newtheorem{cor}{Corollary}
\newtheorem{defi}{Definition}
\newtheorem{lem}{Lemma}
\newtheorem{prop}{Proposition}
\newtheorem{theoremstar}{Theorem}



\title{\textbf{On Multi-Stage Loss Dynamics in Neural Networks: Mechanisms of Plateau and Descent Stages}}

\author{Zheng-An Chen\textsuperscript{\rm 1} , Tao Luo\textsuperscript{\rm 1,2}\footnote{corresponding author, luotao41@sjtu.edu.cn},  
Guihong Wang\textsuperscript{\rm 1}
 \footnote{in alphabetic order}\\
\textsuperscript{\rm 1} School of Mathematical Sciences, Shanghai Jiao Tong University \\
\textsuperscript{\rm 2}Institute of Natural Sciences, MOE-LSC, Shanghai Jiao Tong University; \\ CMA-Shanghai, Shanghai Artificial Intelligence Laboratory\\}

\date{}

\begin{document}

\maketitle

\begin{abstract}
The multi-stage phenomenon in the training loss curves of neural networks has been widely observed, reflecting the non-linearity and complexity inherent in the training process. In this work, we investigate the training dynamics of neural networks (NNs), with particular emphasis on the small initialization regime, identifying three distinct stages observed in the loss curve during training: the initial plateau stage, the initial descent stage, and the secondary plateau stage. Through rigorous analysis, we reveal the underlying challenges contributing to slow training during the plateau stages. While the proof and estimate for the emergence of the initial plateau were established in our previous work, the behaviors of the initial descent and secondary plateau stages had not been explored before. Here, we provide a more detailed proof for the initial plateau, followed by a comprehensive analysis of the initial descent stage dynamics. Furthermore, we examine the factors facilitating the network's ability to overcome the prolonged secondary plateau, supported by both experimental evidence and heuristic reasoning. Finally, to clarify the link between global training trends and local parameter adjustments, we use the Wasserstein distance to track the fine-scale evolution of weight amplitude distribution.
\end{abstract}


\section{Introduction}


In recent years, deep learning has achieved remarkable advancements across various fields \cite{Vaswani2017,dosovitskiy2020image,rajpurkar2018deep,sezer2020financial}. Research on the training process is crucial for understanding and optimizing neural networks (NNs). Unlike traditional machine learning, neural networks exhibit a high-dimensional, complex loss landscape, complicating the optimization process. Researchers have long observed a multi-stage phenomenon in the loss curve, where the loss can plateau and then abruptly decrease without obvious cause \cite{park2000adaptive,fukumizu2000local,saad1995line}, reflecting the complexity and variability of training dynamics. Understanding the mechanisms behind plateau and descent stages is essential for designing more effective parameter optimization strategies in neural networks.

The choice of parameter initialization is of vital importance on both the training dynamics and characteristics of the loss curve. For large initialization, the dynamics of neural networks can be approximately regarded as a kernel regression, known as the neural tangent kernel (NTK) \cite{jacot2018neural,huang2020dynamics,arora2019exact}, where parameters remain close to their initialization throughout training and loss exponentially decreases. Situated between large and small initialization, the gradient flow (GF) of neural networks can be captured by nonlinear partial differential equations (PDEs) from a mean field perspective \cite{mei2018mean,sirignano2020mean,rotskoff2018parameters,chizat2018global}, offering insights into how networks transition between different dynamical regimes. At the other end of the spectrum, small initialization induces a condensation effect, where parameters will   initially aggregate along a few directions during the early stage of training~ \cite{luo2021phase,zhou2022towards,chen2023phase,chen2024three-initial,min2023early}. Subsequently, the dynamics transition through multiple nonlinear stages. This phenomenon is particularly interesting as it reflects a more active exploration of the parameter in the parameter space and is the focus of our series work.

Previous studies on small initialization mainly  focus on the initial stage of training, revealing the condensation phenomenon for multi-layer feedforward neural networks \cite{chen2024three-initial} and convolutional neural networks \cite{zhou2023understanding}. However, these studies largely overlook the characterization of the subsequent stages in the loss curve.

In this work, we conduct a multi-stage analysis of the loss curve and identify three distinct stages in the early training process: the initial plateau stage, the initial descent stage, and the secondary plateau stage, based on their manifestations on the curve. The dynamics governing these stages are driven by different dominant mechanisms, ranging from linear to nonlinear behavior. We provide rigorous theoretical analysis of the dynamics within each stage and use macroscopic quantities to uncover the underlying causes behind both the plateau and descent stages.


\section{Main Results and Related Works}
\subsection{Main results}
In this paper, we consider two-layer NNs, formulated as $f_{\vtheta} (\vx) = \sum_{k=1}^m a_k \sigma(\vw_k^{\T} \vx)$, where $m$ is the width of the neural network and $\vtheta$ including $a_k$'s and $\vw_k$'s are independent and identically distributed (i.i.d.) from normal distributions $N(0,\frac{1}{m^{2\alpha}})$ and $N(0,\frac{1}{m^{2\alpha}}\mI_d)$, respectively. Without loss of generality, we focus on the gradient flow, denoting $T_0=0$ as the initial time of training in this paper. We consider $\alpha>\frac{1}{2}$ which corresponding to the small initialization setting, as mentioned above. Empirically, we have discovered that the early stage of training loss curved can be divided into three distinct stages: the initial plateau stage, the initial descent stage and the secondary plateau stage, which is shown in Fig.~\ref{fig:diff_stage}. The plateau and descent stages are quite common on the loss curve of training \cite{fukumizu2000local}, but need a more comprehensive consideration. In this work, we aim to figure out the dominant dynamics of these stages and provide a rigorous analysis. Before presenting the formal theorems and proofs, we will discuss each of the three stages individually in an intuitive manner (see Theorems~\ref{thm...informal_1}, \ref{thm...informal_2}, \ref{thm...informal_3}, and \ref{thm...informal_4}).

\begin{figure}[t]
    \centering
    \subfigure[]{
    \centering
    \includegraphics[width=0.30\textwidth]{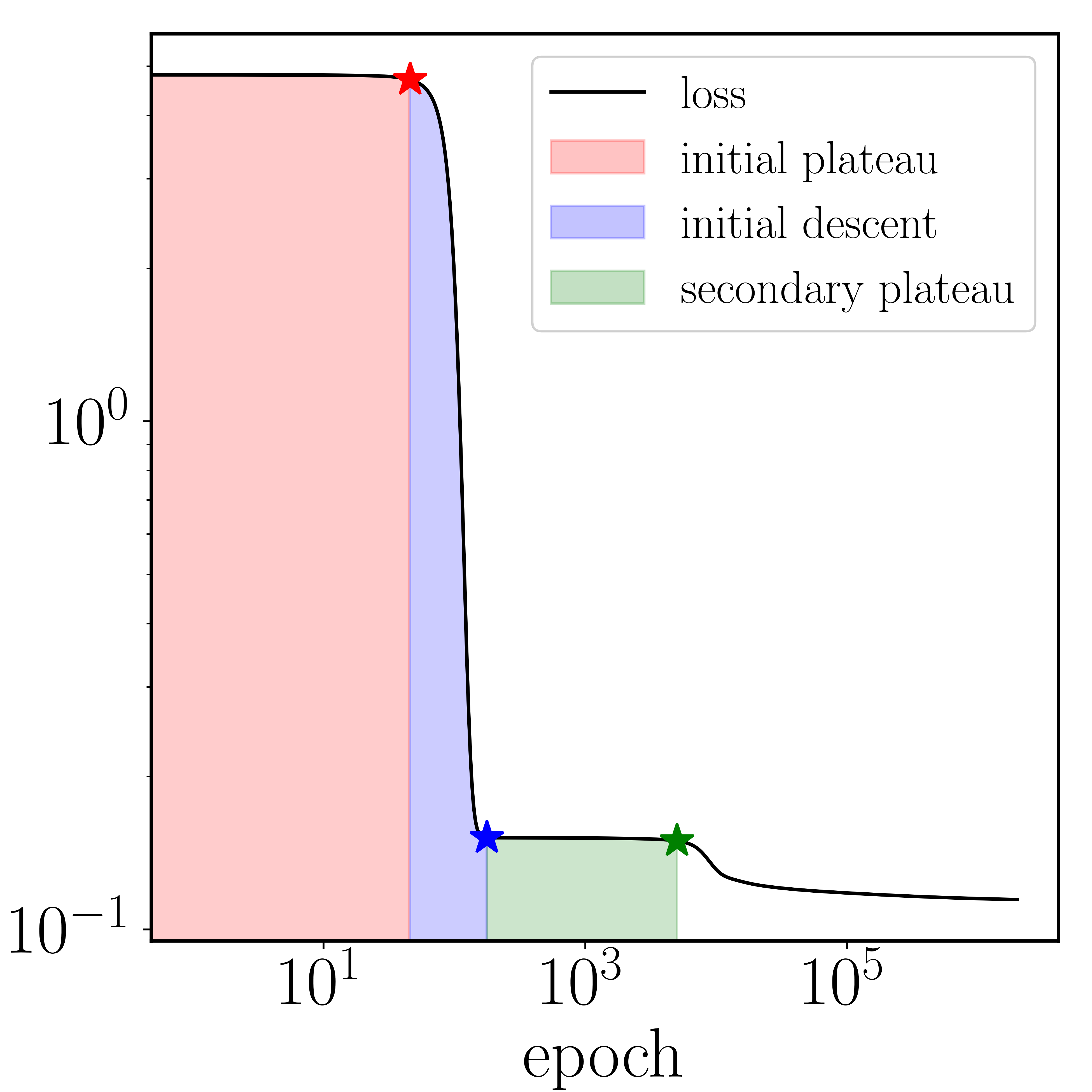}
   }
    \hspace{-5mm}
    \subfigure[]{
    \includegraphics[width=0.30\textwidth]{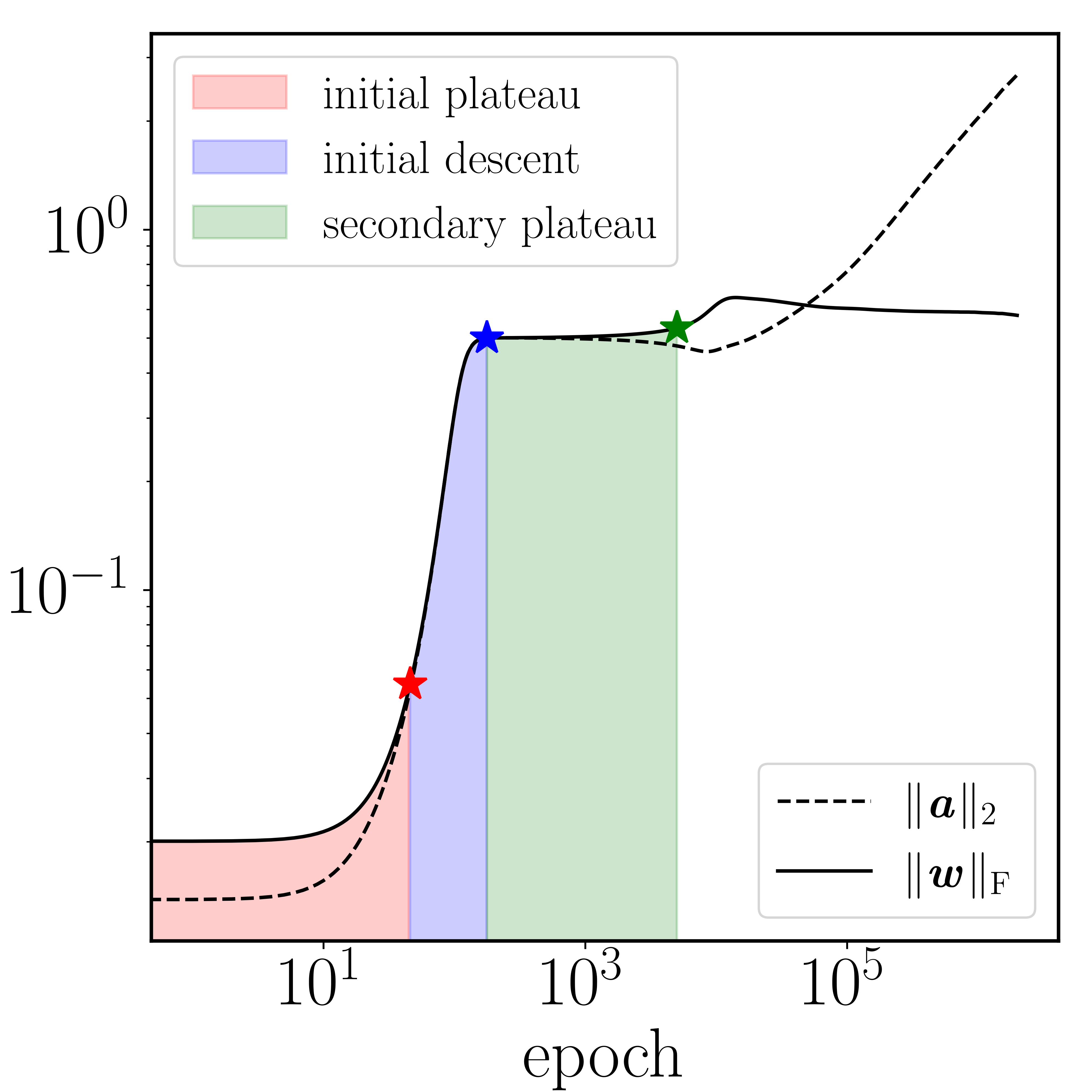}
    }
    \hspace{-5mm}
    \subfigure[]{
    \includegraphics[width=0.30\textwidth]{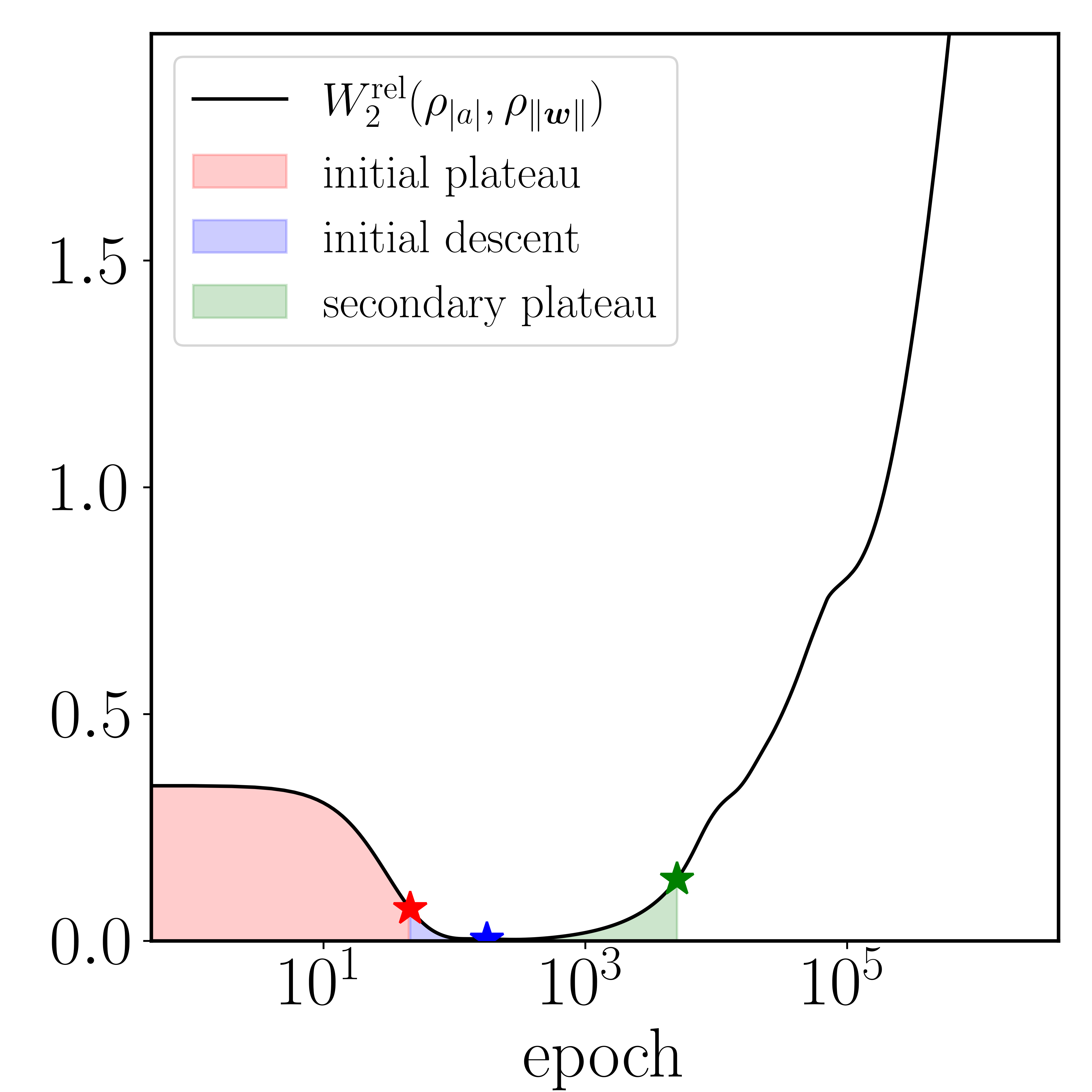}
    }    
    \caption{The behavior of training loss (panel (a)), norms of weights (panel (b)) and relative Wasserstein distance between weights  (panel (c))  that is  $W_2^{\rm{rel}}(\rho_{|a|},\rho_{\norm{\vw}}) =  W_2 (\rho_{|a|},\rho_{\norm{\vw}})/\norm{\rho_{|a|}}_2 $ during training process. The red region is the initial plateau stage, the blue region is the initial descent stage, and the green region is the secondary plateau stage. }
    \label{fig:diff_stage}
\end{figure}

\begin{itemize}[wide, labelwidth=!,labelindent=0pt]
\item \textbf{Initial Plateau Stage}: With small initialization, the loss typically undergoes a plateau stage as shown in Fig.~\ref{fig:diff_stage}(a),  resembling a ``warming up'' period for the parameters. During this period, the parameters behaves more like a two-layer linear network, which has been investigated in  \cite{chen2023phase} and is referred to as the initial condensation. We present a new proof for the behavior of neural networks during this stage, refining previous results. In terms of the loss behavior, we have the following informal theorem.
    \begin{theoremstar}[Informal statement of Theorem \ref{thm::initial_stage}: Initial plateau stage]\label{thm...informal_1}
        For $\alpha>\frac{1}{2}$, with a high probability over the choice of initial parameter $\vtheta(T_0)$, there exists a time $T_{\rm{p}} = \Omega(\log m)
        $ such that the loss decay rate during initial plateau stage is vanishing as $m\to+\infty$, that is, 
        \begin{equation}
            \lim_{m \rightarrow + \infty}         \frac{|R(\vtheta(T_{\rm{p}}))-R(\vtheta(T_0))|}{T_{\rm{p}} - T_0} = 0.
        \end{equation}

    \end{theoremstar}
    In Theorem~\ref{thm...informal_1},  $T_{\rm{p}}$ is the duration of the initial plateau stage which is marked as the red star in Fig.~\ref{fig:diff_stage}. This theorem describes the dependence of $T_{\rm{p}}$ on the width $m$ and provides a characterization of the gradually changing loss. That is, the average decay rate of the loss is very small. For a more detailed statement, refer to Theorem~\ref{thm::initial_stage}.
    
    \item \textbf{Initial Descent Stage}: Following the initial plateau stage, we present a theorem that provides insights into the network’s transition into nonlinear behavior. This analysis highlights how the network's structure begins to evolve during this stage, with dynamics approaching an approximate critical point.
    \begin{theoremstar}[Informal statement of Theorem~\ref{thm::descent}: Initial descent stage]\label{thm...informal_2}
        For $\alpha>\frac{1}{2}$, with a high probability over the choice of initial parameters $\vtheta(T_0)$, there exists a time $T_{\rm{d}}(>T_{\rm{p}})$ such that
        \begin{equation} \label{eq...descent_time_informal}
           \lim_{m \rightarrow \infty} \frac{ T_{\rm{d}}}{\log m} = \frac{2\alpha-1}{2},
        \end{equation}
        and the ratio of loss decay rate between initial descent stage and initial plateau stage diverges as $m\to +\infty$, that is,
        \begin{equation}
          \lim_{m \rightarrow + \infty}  \frac{\left|R (\vtheta(T_{\rm{d}}))- R (\vtheta(T_{\rm{p}})) \right|}{\left| R (\vtheta(T_{\rm{p}})) - R (\vtheta(T_0)) \right|} \frac{\left|T_{\rm{p}}-T_0 \right|}{\left|T_{\rm{d}}-T_{\rm{p}}\right|} = + \infty,
        \end{equation}
        where $T_{\rm{p}}$ is defined in Theorem~\ref{thm...informal_1}.
    \end{theoremstar}
    In Theorem~\ref{thm...informal_2}, $T_{\rm{d}}$ refers to the milestone marking the end of the initial descent stage, represented by the blue star in Fig.~\ref{fig:diff_stage}. This theorem gives a clear description of $T_{\rm{d}}$ which depends solely on $m$ and $\alpha$, as experimentally verified in Section~\ref{sec...exp_descent_time}. Moreover, the decrease in loss during the initial descent stage is at a higher rate compared to the initial plateau stage, resulting in a significant loss descent. A more rigorous statement is provided as Theorem~\ref{thm::descent}.
    
    \item \textbf{Secondary Plateau Stage}: In this stage, the dynamics remain close to an approximate critical point, lingering in its vicinity for a long period. Consequently, the loss curve exhibits another plateau stage.
    \begin{theoremstar}[Informal statement of Theorem~\ref{thm::second_plateau}: Secondary plateau stage]\label{thm...informal_3}
         For $\alpha>\frac{1}{2}$, with a high probability over the choice of initial parameter $\vtheta(T_0)$, there exists a time $T_{\rm{sp}}(>T_{\rm{d}})$ such that
         \begin{equation}
             \lim_{m \rightarrow + \infty} \frac{T_{\rm{sp}}}{\log m} = + \infty,
         \end{equation}
         and the ratio of loss decay rate between secondary plateau stage and initial descent stage is vanishing as $m\to+\infty$, that is,
         \begin{equation}
           \lim_{m \rightarrow + \infty}  \frac{\left| R (\vtheta(T_{\rm{sp}})) - R (\vtheta(T_{\rm{d}})) \right|}{\left|R (\vtheta(T_{\rm{d}}))- R (\vtheta(T_{\rm{p}})) \right|} \frac{\left|T_{\rm{d}}-T_{\rm{p}}\right|}{\left|T_{\rm{sp}}-T_{\rm{d}} \right|} = 0,
         \end{equation}
         where $T_{\rm{d}}$ and $T_{\rm{p}}$ are defined in Theorems~\ref{thm...informal_1} and \ref{thm...informal_2}, respectively.
    \end{theoremstar}
    Here $T_{\rm{sp}}$ represents the milestone marking the end point of the secondary plateau stage, indicated by  the green star in Fig.~\ref{fig:diff_stage}. Theorem~\ref{thm...informal_3} illustrates the plateau effect that its loss decay rate is significantly less than that of the initial descent stage, resulting in a slow change in loss.
\end{itemize}

Moreover, we dig into the behavior of parameters and introduce the Wasserstein distance as a novel tool to describe the microscopic evolution of different layer's weight distribution during training. Here we state an informal theorem based on results of Corollaries~\ref{cor::similar_fp} and \ref{cor::descent_distribution} to show that different layers of weights converge to the same distribution during the training.

\begin{theoremstar}(Informal statement of Corollaries~\ref{cor::similar_fp} and \ref{cor::descent_distribution}: Amplitude distribution of weights are similar)\label{thm...informal_4}
    For $\alpha>\frac{1}{2}$, with a high probability over the choice of initial parameter $\vtheta(T_0)$, we have for any $t \in [T_{\rm{p}}, T_{\rm{d}}] $
    \begin{equation}
        \lim_{m \rightarrow +\infty} \frac{W_2 (\rho_{|a|}, \rho_{\norm{\vw}})}{W_2 (\rho_{|a|})} = 0,
    \end{equation}
    where $T_{\rm{d}}$ and $T_{\rm{p}}$ are defined in Theorems~\ref{thm...informal_1} and \ref{thm...informal_2}, respectively.
\end{theoremstar}
Fig.~\ref{fig:diff_stage}(c) shows that the relative Wasserstein distance declines in the initial descent stage where the norms of the weights are in the same level as shown in Fig.~\ref{fig:diff_stage}(b). Moreover, using experimental data and heuristic reasoning, we explore the factors that allow the network to exit the secondary plateau and offer intuition into the mechanisms driving the network’s evolution beyond this stage in Section~\ref{sec...plateau_departure}.

By quantifying both global trends and local changes in neural parameter behavior, our results highlight the importance of understanding neural network dynamics across different stages. These insights can offer practical implications for improving network generalization and optimization strategies. Theorems~\ref{thm...informal_1}, \ref{thm...informal_2}, \ref{thm...informal_3}, and \ref{thm...informal_4} and alignment of the milestones in the subgraph of Fig.~\ref{fig:diff_stage} tells a coherent story both theoretically and experimentally: During training under small initialization, the loss changes gradually while the parameters is moving towards in the same direction. Once they achieve a similar distribution, the loss begins to decrease until the dynamics approaches an approximate critical point and settles into a secondary plateau. As the weights strive to separate, the loss leaves the plateau and begins to decrease once again. 

\subsection{Related Works}\label{sec...related_works}
In this section, we discuss some related works. In fact, our nonlinear model can be viewed as a perturbation of a homogeneous linear model in its early stages. The homogeneous model has been widely applied in analyzing both the early stage and final stage dynamics of parameter behavior.
\begin{itemize}[wide, labelwidth=!,labelindent=0pt]
    \item \textbf{Early stage analysis}: There has been some work on early stage analysis. In matrix factorization problems, the model can be viewed as shallow or deep linear networks, though the problem remains non-convex.  Ye and Du \cite{ye2021global} showed that randomly initialized gradient descent converges to a global minimum of the asymmetric low-rank factorization problem at a polynomial rate. However, this model is linear. To some extent, our work builds on these approaches by extending them to nonlinear models. Studies such as  \cite{chen2023phase,chen2024three-initial,kumar2024early} characterize the initial stage where parameters condense or their directions converge in nonlinear models, though they do not fully capture the subsequent stages.  Min et al.~\cite{min2023early} identifies different stages for classification problems with separable data. However, due to certain limitations in the data and the structure of the loss function, it does not yet fully account for the secondary plateau stage and the transition from linear to nonlinear models.

    \item \textbf{Final stage analysis}: Other notable works focus on the convergence properties of homogeneous models in the final stages of classification where researchers often start their work under the assumption of correct classification. Lyu and Li~\cite{lyu2019gradient} modelled the final state as a max-margin optimization problem and get the convergence of both the normalized margin and its smoothed version to the objective value at a KKT point. Furthermore, Ji and Telgarsky~\cite{ji2020directional} extended the previous results and achieves directional convergence outcomes. Although the problems may seem different, as demonstrated in  \cite{kumar2024directional}, we believe that these two categories are fundamentally related. 
\end{itemize} 

The rest of the paper is organized as follows: Section~\ref{sec...preliminaries} presents preliminaries including the introduction of notation, the problem formulation, and key insights of different stages. In Section~\ref{sec...ip} we give a new proof for initial plateau stage. In Section~\ref{sec...id} we characterize the initial descent time both theoretically and experimentally for initial descent stage. In Section~\ref{sec...sp}, we show the long-term existence for secondary plateau stage and give an insight of the departure of the plateau. Finally, Section~\ref{sec...summary} is the summary of the paper. The appendix contain propositions along with detailed experimental results.

\section{Preliminaries}\label{sec...preliminaries}
\subsection{Notation}
First, we introduce some notation that will be used throughout this paper. Let $m$ be the width of hidden layers. Let $[n]$ denote the set of integers from $1$ to $n$ and $[k:n]$ be the set of integers from $k$ to $n$. Denote vector $L^2$ norm as $\norm{\cdot}_2$.  For a vector $\bm{v}$, denote its $k$-th entry as $v^k$. We use $O(\cdot)$, $\Omega(\cdot)$, and $\Theta(\cdot)$ for the standard Big-O, Big-Omega , and Big-Theta notation. For other key notation, we can refer to the Appendix~\ref{app...notation}.

\subsection{Problem formulation}
We consider two-layer neural networks with smooth activation function. We believe that linear network will approximate the nonlinear network well under small initialization setting until some milestone. The main issue is how to find this milestone and what to do after this. In the process of exploring such questions, we will gradually deepen our understanding of key phenomena such as condensation phenomena and the plateau phenomenon in the loss curve.

The model we now consider is 
\begin{equation*}
    f_{\vtheta} (\vx) = \sum_{k=1}^m a_k \sigma(\vw_k^{\T} \vx),
\end{equation*}
with initial parameter $a_k(0) \sim N(0,\frac{1}{m^{2 \alpha}})$ and $\vw_k(0) \sim N(0,\frac{1}{m^{2 \alpha}} \mI_d)$. For simplicity, we consider the following population risk function (also known as population loss function):
\begin{equation*}
    R (\vtheta) = \frac{1}{2} \int_{\sR^d} \left( f_{\vtheta} (\vx) - f(\vx) \right)^2 \rho(\vx) \diff{\vx}.
\end{equation*}
Without loss of generality, we make the following assumptions.
\begin{assump}[Symmetric sampling probability]\label{assump..SymSamp}
    Density function $\rho(\vx)$ has compact support. Moreover, $\int_{\sR^d} x_i^2 \rho(\vx) \diff{\vx} = 1$ for $i \in [d]$ and $\int_{\sR^d} x_i x_j \rho(\vx) \diff{\vx} = 0$ for $i \neq j$, $i,j\in [d]$.
\end{assump}

\begin{assump}[Non-vanishing leading term]\label{assump..Non-degenDirect}
    The initial condensation direction is non-degenerate. That is 
    \begin{equation*}
        \int_{\sR^d} f(\vx) \vx \rho(\vx) \diff{\vx} = (1, 0, \dots, 0)^{\T}.
\end{equation*}
\end{assump}
\begin{rmk}
    We only use the symmetry of the distribution. The result can be applied to situations with very few sample data, in which $R$ is the empirical loss/risk and $\rho$ is the sum of delta functions. Our result also can be extended to the networks with bias term $b_k$'s. 
\end{rmk}
We assume our activation function looks like tanh function, which includes a class of smooth activation functions. Without loss of generality, we take the following assumption.
\begin{assump}[Tanh-like activation function]\label{assump..TanhActivation}
    The activation function $\sigma$ satisfies: $\sigma(0) = 0, 
    \sigma^{(1)}(0)  = 1, 
    \sigma^{(2)}(0)  = 0, 
    |\sigma^{(3)}(z)|  \le C_{\rm{L}}$ for all $z\in\sR$.
\end{assump}
We suppose Assumptions~\ref{assump..SymSamp}, \ref{assump..Non-degenDirect}, and \ref{assump..TanhActivation} hold throughout this paper. Based on them, we approximate gradient descent using gradient flow under the condition that the learning rate is sufficiently small. Consequently, the dynamical system reads as follows (with the detailed derivation provided in the Appendix).
\begin{equation}
\label{equ::grad_descent}
\left\{
    \begin{aligned}
        \frac{\D a_k}{\D t} &= w_k^1 - \left( \sum_{l=1}^m a_l w_l^1 \right) w_k^1 - \sum_{i=2}^d \left( \sum_{l=1}^m a_l w_l^i \right) w_k^i + f_k,\\
        \frac{\D \vw_k}{\D t} &=  -\int_{\sR^d} \left(\sum_{l=1}^m a_l\left(\vw_l^{\T} \vx \right)-f(\vx)\right) a_k \vx \rho(\vx) \diff{\vx} + \vg_k, \\
    \end{aligned}
\right.
\end{equation}
where the superscript in $w_k^i$ means the $i$-th entry. Here the higher-order term $f_k$ and $\vg_k$ are defined as follows:
\begin{align}
    f_k & = \int_{\sR^d} f(\vx) \left( \frac{1}{3!} \sigma^{(3)} (\xi_k(\vx)) (\vw_k^{\T} \vx)^3 \right) \rho(\vx) \diff{\vx}\nonumber \\
    &\quad - \int_{\sR^d} \left( \sum_{l=1}^m \frac{1}{3!} \sigma^{(3)} (\xi_l(\vx)) a_l (\vw_l^{\T} \vx)^3  \right) (\vw_k^{\T} \vx) \rho(\vx) \diff{\vx}\nonumber \\
    & \quad - \int_{\sR^d} \left( \sum_{l=1}^m a_l \left( \vw_l^{\T} \vx + \frac{1}{3!} \sigma^{(3)}(\xi_l(\vx)) (\vw_l^{\T} \vx)^3\right) \right) \left( \frac{1}{3!} \sigma^{(3)} (\xi_k(\vx)) (\vw_k^{\T} \vx)^3 \right)  \rho(\vx) \diff{\vx}.\\
    \vg_k &=\int_{\sR^d} f(\vx) a_k\left(\frac{1}{2!} \sigma^{(3)}\left(\eta_k (\vx)\right)\left(\vw_k^{\T} \vx\right)^2\right) \vx \rho(\vx) \diff{\vx}\nonumber  \\
    & \quad -\int_{\sR^d} \left(\sum_{l=1}^m a_l \left(\vw_l^{\T} \vx\right)\right) a_k \left( \frac{1}{2!} \sigma^{(3)}\left(\eta_k (\vx)\right)\left( \vw_k^{\T} \vx \right)^2 \right) \vx \rho(\vx) \diff{\vx}]\nonumber \\
    & \quad -\int_{\sR^d} \left(\sum_{l=1}^m a_l \frac{1}{3!} \sigma^{(3)}\left(\xi_l (\vx)\right)\left(\vw_l^{\T} \vx\right)^3\right) a_k\left(1+\frac{1}{2!} \sigma^{(3)}\left(\eta_k (\vx)\right)\left(\vw_k^{\T} \vx \right)^2\right) \vx \rho(\vx) \diff{\vx}.
\end{align}
where $\xi_k (\vx)$ and $\eta_k (\vx)$ come from Taylor's expansion. To better reveal the structure of the dynamics, we can express $\vw_k$ in entry-wise form.
\begin{equation}
\left\{
\label{equ::whole_dynamics}
    \begin{aligned}
        \frac{\D a_k}{\D t} &= w_k^1 - \left( \sum_{l=1}^m a_l w_l^1 \right) w_k^1 - \sum_{i=2}^d \left( \sum_{l=1}^m a_l w_l^i \right) w_k^i + f_k,\\
        \frac{\D w_k^1}{\D t} & = a_k - \left( \sum_{l=1}^m a_l w_l^1 \right) a_k +g_k^1, \\
        \frac{\D w_k^i}{\D t} & = -\left( \sum_{l=1}^m a_l w_l^i \right) a_k + g_k^i. \\
    \end{aligned}
\right.
\end{equation}

We will estimate the higher-orders terms carefully throughout this paper. To achieve this, we introduce neuron-wise $L^{\infty}$ norm, denoted as  $q_{\rm{max}}$:
\begin{equation}
\label{defi::q_max}
    q_{\rm{max}}(t) = \max_{k \in [m],i \in [d]} \{ |a_k(t)|,|w_k^i(t)|\}.
\end{equation}
By direct calculation, it can be shown that these terms are negligible in the stages we are concerned with. This is formalized in the following lemma.
\begin{lem}[Estimate of higher order terms]
\label{lem::higher_term}
The following inequalities hold for $q_{\text{max}} \le 1$:    
\begin{equation}
\label{equ::higher-order_control}
    |f_k|, \norm{\vg_k}_2 \lesssim q_{\rm{max}}^3 + m q_{\rm{max}}^5.
\end{equation}
\end{lem}

\begin{proof}
    Let $C_1 = \sup_{\vx \in \rm{supp}(\rho)} \left| f(\vx) \right|$ and $C_2 = \diam(\rm{supp}(\rho))$, where $C_2$ denotes the diameter of $\rm{supp}(\rho)$. Since density function $\rho(\vx)$ has compact support, $C_1$ and $C_2$ are finite. We will provied the proof for$f_k$ and the case for $\vg_k$ is similar. By definition of $f_k$, we have
    \begin{align*}
        |f_k| & = \left|\int_{\sR^d} f(\vx) \left( \frac{1}{3!} \sigma^{(3)} (\xi_k(\vx)) (\vw_k^{\T} \vx)^3 \right) \rho(\vx) \diff{\vx}\right. \\
        &\quad - \int_{\sR^d} \left( \sum_{l=1}^m \frac{1}{3!} \sigma^{(3)} (\xi_l(\vx)) a_l (\vw_l^{\T} \vx)^3  \right) (\vw_k^{\T} \vx) \rho(\vx) \diff{\vx} \\
        &\quad - \left.\int_{\sR^d} \left( \sum_{l=1}^m a_l \left( \vw_l^{\T} \vx + \frac{1}{3!} \sigma^{(3)}(\xi_l(\vx)) (\vw_l^{\T} \vx)^3\right) \right) \left( \frac{1}{3!} \sigma^{(3)} (\xi_k(\vx)) (\vw_k^{\T} \vx)^3 \right)  \rho(\vx) \diff{\vx} \right|\\
        &\le \frac{1}{3!} C_{\rm{L}} C_1 C_2^3 d^3 q_{\rm{max}}^3 + \frac{1}{3!} C_{\rm{L}} C_2^4 d^4 m q_{\rm{max}}^5 + \frac{1}{3!} C_{\rm{L}}C_2^4 d^4 \left( 1+ \frac{1}{3!} C_{\rm{L}}  C_2^2 d^2 \right) m q_{\max}^5 \\
        &\lesssim q_{\rm{max}}^3 + m q_{\rm{max}}^5,
    \end{align*}
    where the condition $q_{\rm{max}} \le 1$ is applied to estimate the third term.
\end{proof}

\subsection{Separation of different stages: an intuitive discussion}

According to the behavior of loss during the training process, we divide the early stage of training into three stages: the initial plateau stage, the initial descent stage, and the secondary plateau stage, as shown in the Fig.~\ref{fig:diff_stage}(a). These three stages constitute the primary focus of this paper.

We introduce the Wasserstein distance to characterize difference between the empirical distribution of input weights and output weights. The Wasserstein distance is defined as
\begin{equation}\label{defi....wassersteindist}
    W_2 (\rho_1, \rho_2) = \inf_{\gamma \in \fC (\rho_1, \rho_2)}  \left( \int \norm{\vtheta_1 - \vtheta_2}_2^2 \gamma(\D \vtheta_1, \D \vtheta_2) \right)^{\frac{1}{2}},
\end{equation}
where the infimum is taken over all couplings of $\rho_1$ and $\rho_2$. That is $\gamma$ is a joint probability distribution defined on $\sR^m \times \sR^m$, such that its marginal distributions are $\rho_1$ and $\rho_2$, respectively. We consider empirical distribution $\rho_{|a|} = \frac{1}{m} \sum_{k=1}^m \delta_{|a_k|}$ and $\rho_{\norm{\vw}} = \frac{1}{m} \sum_{k=1}^m \delta_{\norm{\vw_k}}$. There is an upper bound we will use in the following discussion.
\begin{equation}
    W_2 (\rho_{|a|},\rho_{\norm{\vw}}) \le \left( \frac{1}{m} \sum_{k=1}^m \left(|a_k| -\norm{\vw_k} \right)^2 \right)^{\frac{1}{2}}.
\end{equation}
As shown in Fig.\ref{fig:diff_stage}, during the three stages of the training loss curve, the norms of the weights and their relative Wasserstein distance change simultaneously. On the one hand, the norms of ${\vw}$ and ${\va}$ progressively approach equality at the red star in Fig.\ref{fig:diff_stage} which is the beginning of the initial descent stage. On the other hand, at the end of the initial descent stage (the blue star), both the norms of the weights seem to settle onto the same plateau and the relative Wasserstein distance vanishes. Fig.\ref{fig:diff_stage} illustrates that the phenomenon depicted, from the left panel to the right, represents observations of the same process at different scales, transitioning from macroscopic to microscopic views. 
In this paper, we will theoretically analyze the dynamical behavior of the weights during each period. Here, we outline the main ideas behind the proof. 

In the initial stage, due to the very small initialization, the parameters will be approximated by the linear dynamics described below. As a result, the weight vector will be proved to condense in a specific direction starting from a randomly initialized state, with careful error estimate between the original dynamics~\eqref{equ::whole_dynamics} and linearized dynamics~\eqref{equ::initial_stage}. We have established a similar result in  \cite{chen2023phase}. But in the present work, we employ a different technique to obtain improved estimates.
\begin{equation}
\label{equ::initial_stage}
\left\{
    \begin{aligned}
        \frac{\D a_k}{\D t} &= w_k^1, \\
        \frac{\D w_k^1}{\D t} &= a_k, \\
        \frac{\D w_k^i}{\D t} &= 0, \quad i = 2, \dots, d.
    \end{aligned}
\right.
\end{equation}

In the initial descent stage, after considering higher order approximation, the dynamics of parameters will be governed by the following system:
\begin{equation}
\label{equ::descent_stage}
\left\{
    \begin{aligned}
        \frac{\D a_k}{\D t} &= w_k^1 -  \left( \sum_{l=1}^m a_l w_l^1 \right) w_k^1 , \\
        \frac{\D w_k^1}{\D t} &= a_k - \left( \sum_{l=1}^m a_l w_l^1 \right) a_k, \\
        \frac{\D w_k^i}{\D t} &= 0,\quad i=2, \dots, d.
    \end{aligned}
\right.
\end{equation}
We will study the original dynamics~\eqref{equ::whole_dynamics} by making use of the analysis on Eq.~\eqref{equ::descent_stage}. To do this, we introduce
\begin{align}
    K & = \sum_{k=1}^m a_k w_k^1, \label{defi::K} \\
    K'& = \sum_{k=1}^m  \left( a_k^2 + (w_k^1)^2 \right).\label{defi::K'}
\end{align}
Note that $K$ approximately solves the equation
\begin{equation}
\label{equ::K_dynamics}
    \frac{\D K}{\D t} = 2K(1-K).
\end{equation}
Clearly, Eq.~\eqref{equ::K_dynamics} admits an explicit solution:
\begin{equation}
    K(t) = \frac{C(t_0) \E^{2(t-t_0)}}{1 + C(t_0) \E^{2(t-t_0)}},
\end{equation}
where $C(t_0) = \frac{K(t_0)}{1-K(t_0)}$ when $K(t_0) \neq 1$. Thus the dynamics in Eq.~\eqref{equ::K_dynamics} will rapidly converge to the limit $K = 1$.

In the secondary plateau stage, since the dynamics has converged to a critical point, the entire dynamics enter a prolonged plateau stage to explore new directions. This is reflected on the loss curve as an extended plateau. At this point, the effective dynamics is dominated by the following dynamics. We start from a ``critical point'' corresponding to the condition $K = 1$ and then the leading term vanishes because
\begin{align*}
    w_k^1 - \left( \sum_{l=1}^m a_l w_l^1 \right) w_k^1 &= (1-K) w_k^1 \approx 0, \\ 
    a_k - \left( \sum_{l=1}^m a_l w_l^1 \right) a_k &= (1-K) a_k \approx 0.
\end{align*}
Thus, the evolution of the parameters at this time can be approximated by the following system:
\begin{equation}
 \label{equ::sec_dynamics}
 \left\{
     \begin{aligned}
        \frac{\D a_k}{\D t} &=  - \sum_{i=2}^d \left( \sum_{l=1}^m a_l w_l^i \right) w_k^i + f_k,\\
        \frac{\D w_k^1}{\D t} & =g_k^1, \\
        \frac{\D w_k^i}{\D t} & = -\left( \sum_{l=1}^m a_l w_l^i \right) a_k + g_k^i, \quad i= 2, \dots, d. \\
     \end{aligned}
\right.
\end{equation}
Throughout the proof of the three stages, we will find the following lemma be particularly useful, as it separates $w_k^1$ from the other directions $w_k^i$ and decouples the original complicated dynamics \eqref{equ::whole_dynamics} into simple form in each stage.  
\begin{lem}[Growth lemma]
    \label{lem::growth}
    Let $T_{\rm{test}} = \left\{ t \ge 0 ~|~ q_{\rm{max}} \le \frac{1}{m^{\beta}} (\log m)^\gamma \right\}$ where $\beta, \gamma>0$ such that $\frac{1}{m^{\beta}} (\log m)^\gamma \le 1$. We have for $i=2, \dots, d$, $t_0 \le t \le T_{\rm{test}}$
    \begin{equation}\label{eq::growthlemma}
        \left( \sum_{k=1}^m \left( w_k^i (t) \right)^2  \right)^{\frac{1}{2}} \le \left( \sum_{k=1}^m \left( w_k^i (t_0) \right)^2 \right)^{\frac{1}{2}} + C \max \left\{ \frac{(\log m)^{3\gamma}}{m^{3\beta-\frac{1}{2} }}, \frac{(\log m)^{5\gamma}}{m^{5\beta -\frac{3}{2}}}\right\} (t-t_0).  
    \end{equation}
\end{lem}

\begin{rmk}
    The time $T_{\rm{test}}$ serves as an intuitive hypothesis regarding the neuron-wise $L^{\infty}$ norm $q_{\rm{max}}$ at various stages.
\end{rmk}

\begin{proof}
Through direct calculation, we find that the leading term is negative, which provides an upper bound,
    \begin{align*}
        \frac{\D \left( \sum_{k=1}^m (w_k^i)^2\right)}{\D t} &= 2 \sum_{k=1}^m w_k^i \left(  -  a_k \left(\sum_{k=1}^m a_k w_k^i \right) + g_k^i \right) \\
        & \le 2 \left( \sum_{k=1}^m (w_k^i)^2 \right)^\frac{1}{2} \left( \sum_{k=1}^m (g_k^i)^2 \right)^\frac{1}{2}.
    \end{align*}
    By Lemma~\ref{lem::higher_term} and the definition of $T_{\text{test}}$, we have 
    \begin{align*}
        \left( \sum_{k=1}^m (g_k^i)^2 \right)^\frac{1}{2} & \le C \sqrt{m} \max \left\{ q_{\text{max}}^3, m q_{\text{max}}^5  \right\} \\
        &\le C \max \left\{ \frac{(\log m)^{3\gamma}}{m^{3\beta-\frac{1}{2} }}, \frac{(\log m)^{5 \gamma}}{m^{5\beta -\frac{3}{2}}}\right\}.
    \end{align*}
    As a result, we obtain Eq.~\eqref{eq::growthlemma}.
\end{proof}

\section{Initial Plateau Stage}\label{sec...ip}
In this section, we characterize the dynamical behavior during the initial plateau stage. By proving that the linearized dynamics \eqref{equ::initial_stage} closely approximates the original dynamics \eqref{equ::whole_dynamics} in this stage, we illustrate how initially scattered parameters begin converging along a few key directions, effectively acting as a form of parameter reinitialization. This result extends our previous findings~\cite{chen2023phase}, though our novel approach provides more precise insights. The improvements in our results are primarily due to Lemma~\ref{lem::growth}.

Let $\tilde{a}_k$ and $\tilde{w}_k^i$ denote the solution to Eq.~\eqref{equ::initial_stage} with the same initial conditions as those in the original dynamics~\eqref{equ::whole_dynamics}. It has the analytical solution as follows,
\begin{equation}
    \label{equ::analytical_sol}
    \left\{
    \begin{aligned}
        \tilde{a}_k(t) &=  \tfrac{1}{2} (a_k(0) + w_k^1 (0)) \E^t + \tfrac{1}{2} (a_k(0)- w_k^1 (0)) \E^{-t}, \\
        \tilde{w}_k^1(t) &=  \tfrac{1}{2} (a_k(0) + w_k^1 (0)) \E^t - \tfrac{1}{2} (a_k(0)- w_k^1 (0)) \E^{-t},\\
        \tilde{w}_k^i(t) &= w_k^i(0),\quad i=2,\dots,d.
    \end{aligned}
    \right.
\end{equation}
Similarly, we define $\tilde{q}_{\text{max}}$ as follows.
\begin{equation}\label{defi...q_max_tilde}
    \tilde{q}_{\text{max}}  =  \max_{k \in [m],i \in [d]} \{ |\tilde{a}_k(t)|,|\tilde{w}_k^i(t)|\}.
\end{equation}

Denote the difference between original dynamics~\eqref{equ::whole_dynamics} and the linearized dynamics~\eqref{equ::initial_stage} as 
\begin{equation}\label{defi...rmax}
    r_{\rm{max}}(t) = \max_{k \in [m]}  \left\{ |a_k(t) - \tilde{a}_k(t)| + |w_k^1(t) -\tilde{w}_k^1(t)| \right\}.
\end{equation}
We use the following test time $T_1$ to characterize the effective time of linearization.
\begin{equation}\label{defi...T_1}
    T_1 = \inf \left\{ t \ge 0 ~ | ~ q_{\rm{max}} \ge \frac{1}{m^{\alpha_1}} \log m, r_{\rm{max}} \ge \frac{1}{m^{\gamma_1}} (\log m)^{3} \right\},
\end{equation}
where $q_{\rm{max}}$ defined in Eq.~\eqref{defi::q_max} and $\frac{1}{2} < \alpha_1 < \gamma_1$. Taking hyperparameters as in Tab.~\ref{tab::hyper}, we have the main theorem for the initial plateau stage.
\begin{thm}[Initial plateau stage] 
\label{thm::initial_stage}
Let $(\alpha_1, \gamma_1) = \left(\frac{1}{2}\alpha + \frac{1}{4}, \frac{3}{2}\alpha - \frac{1}{4}\right)$ and $T_{\rm{p}} = \frac{2\alpha - 1}{4} \log m$ for $\frac{1}{2} < \alpha \leq \frac{3}{2}$, while $(\alpha_1, \gamma_1) = (1, 2)$ and $T_{\rm{p}} = (\alpha - 1) \log m$ for $\alpha > \frac{3}{2}$. Given $\delta \in (0,1)$, for sufficiently large $m$, with probability at least $1 - \delta$ over the choice of initial parameters $\vtheta(T_0)$, the following holds:
\begin{enumerate}[label = (\roman*).,wide, labelwidth=!,labelindent=0pt]
\item
    Neuron-wise estimate is maintained during the initial plateau stage:
    \begin{equation}
        T_{\rm{p}} \leq T_1.
    \end{equation}
    Furthermore, for $i = 2, \dots, d$, we have 
    \begin{equation}
    \label{equ::initial_w_k^i}
        \max_{k \in m, i \in [2:d]} \left\{ |w_k^i(T_{\rm{p}})| \right\} \lesssim \max \left\{ \frac{1}{m^\alpha} \log m, \frac{1}{m^3} (\log m)^6 \right\}.
    \end{equation}
\item
    The loss function exhibits negligible decay during the initial plateau stage:
    \begin{equation}
    \frac{|R(\vtheta(T_{\rm{p}})) - R(\vtheta(0))|}{T_{\rm{p}}} \lesssim \frac{1}{m^{2\alpha_1-1} \log m}.
    \end{equation}
\end{enumerate}
\end{thm}

\begin{rmk}
    In Theorem~\ref{thm::initial_stage}, we work in a probabilistic setting, i.e., with probability $1-\delta$ the results hold, which will recur frequently in the subsequent conclusions. This is because we adopt random initialization, making the worst-case hard to avoid. However, for sufficiently large $m$, $\delta$ can be sufficiently small, we can see this in Lemma~\ref{lem...init_est}.
\end{rmk}

\begin{rmk}
In Tab.~\ref{tab::hyper}, we present the values of $\alpha_1,\gamma_1$ and $T_{\rm{p}}$ for different $\alpha$. Here, $\alpha_1$ governs the leading order behavior of $q_{\rm{max}}$, while $\gamma_1$ dictates the leading order of $r_{\rm{max}}$. Note that the division of the interval $(\frac{1}{2},+\infty)$ into two subintervals is purely for the convenience of the subsequent proofs and does not indicate any intrinsic distinction between the two cases.
    \begin{table}[h!]
    \centering
    \caption{Hyperparameters in Theorem~\ref{thm::initial_stage}}
    \begin{adjustbox}{max width=\textwidth}
    \begin{tabular}{@{}>{\bfseries}lccccc@{}}
        \toprule
        & $\alpha_1$ & $\gamma_1$  & $T_{\rm{p}}$ \\
        \midrule
        $\frac{1}{2} < \alpha \le \frac{3}{2}$ & $\frac{1}{2} \alpha + \frac{1}{4}$ & $\frac{3}{2} \alpha - \frac{1}{4}$ & $\frac{2\alpha-1}{4} \log m$ \\
        $\frac{3}{2} < \alpha < \infty$& $1$ & $2$ & $(\alpha-1) \log m$ \\
        \bottomrule
    \end{tabular}
    \end{adjustbox}
    \captionsetup{justification=centering}
    \label{tab::hyper}
\end{table}
\end{rmk}
\begin{rmk}
    It should be noted that the existence of $\log m$ term is due the Gaussian initialization. This initialization method produced a relative factor of the order $\log m$, which hinders us from obtaining further results through particle estimation. Fortunately, if we turn to a fine estimate, we can obtain further results.
\end{rmk}

To prove the theorem, we first estimate $\sum_{k=1}^m (w_k^i)^2$ for $i = 2,\dots, d$ in Lemma~\ref{lem::growth}. This quantity represents the contribution of weights in non-condensed directions.

\begin{prop}[Estimate of weights in non-condensed directions at initial plateau stage]
\label{prop::initial_control}
    Let $\alpha_1,\gamma_1$ be defined as in Theorem~\ref{thm::initial_stage}.
    Given $\delta \in (0,1)$, for sufficiently large $m$, with probability at least $1 - \delta$ over the choice of initial parameter $\vtheta(T_0)$, for $0 \le t \le \min \{ T_{\rm{p}}, T_1 \}$ and $i=2,\dots,d$, we have
    \begin{equation}
    \label{equ::ini_higher-order}
         \left( \sum_{k=1}^m (w_k^i)^2  \right)^{\frac{1}{2}} \le C \max \left\{ \frac{1}{m^{\frac{2\alpha-1}{2}}}, \frac{(\log m)^{5}}{m^{\frac{5}{2}}} \right\}. 
    \end{equation}
\end{prop}
\begin{proof}
We just prove the result for $\frac{1}{2} < \alpha \le \frac{3}{2}$ and the other case is similar. By the growth lemma (Lemma~\ref{lem::growth}) and Proposition~\ref{prop::ULB_of_ini_param}, 
\begin{align*}
    \left( \sum_{k=1}^m (w_k^i (t))^2 \right)^\frac{1}{2} &\le  \left( \sum_{k=1}^m (w_k^i (0))^2 \right)^\frac{1}{2} +  C \max \left\{ \frac{(\log m)^3}{m^{3 \alpha_1-\frac{1}{2} }}, \frac{(\log m)^{5}}{m^{5 \alpha_1 -\frac{3}{2}}}\right\} t \\ 
    & \le \frac{C}{m^{\frac{2\alpha-1}{2}}} + C  \frac{1}{m^{\frac{3}{2}\alpha+\frac{1}{4}}} (\log m)^4. \\
\end{align*}
Since $\frac{1}{2} < \alpha \le \frac{3}{2}$, $C  \frac{1}{m^{\frac{3}{2}\alpha+\frac{1}{4}}} (\log m)^4 \le \frac{1}{m^{\frac{2\alpha-1}{2}}}$, the inequality~\eqref{equ::ini_higher-order} holds.
\end{proof}
\begin{rmk}
By comparing Eqs.~\eqref{equ::init_w_k^i} and \eqref{equ::ini_higher-order}, we observe that Proposition \ref{prop::initial_control} effectively demonstrates that each $w_k^i$ shares an equal portion of $\left( \sum_{k=1}^m (w_k^i)^2 \right)$, with each $w_k^i$ scaled by a factor of approximate $\frac{1}{\sqrt{m}}$. In other words, Theorem \ref{thm::initial_stage} derives an $L^{\infty}$ bound from an $L^2$ control.
\end{rmk}

Now we are ready to prove the first part of Theorem~\ref{thm::initial_stage}.
\begin{proof}[Proof of Theorem~\ref{thm::initial_stage} (Part 1)]
The proof of the theorem can be divided into two parts. First, we use Gronwall's inequality to establish a standard error estimate. Then, we apply proof by contradiction to obtain neuron-wise control during the initial plateau stage.
\begin{enumerate}[label = (\roman*).,wide, labelwidth=!,labelindent=0pt]
\item We measure the difference between the linearized dynamics~\eqref{equ::initial_stage} and the original dynamics~\eqref{equ::whole_dynamics}. Define $F_k$ and $G_k^i$ as follows
    \begin{align*}
        F_k & = -\sum_{i=1}^d w_k^i \left( \sum_{l=1}^m a_l w_l^i \right) + f_k,\\
        G_k^i & = - a_k \left( \sum_{l=1}^m a_l w_l^i\right) + g_k^i.
    \end{align*}
By definitions of $F_k$ and $G_k^i$, the following inequalities hold 
\begin{align*}
    \left|a_k(t)-\tilde{a}_k(t)\right| &\le \int_0^t\left|\vw_k^{1}-\tilde{\vw}_k^{1}\right| \diff{s}+\int_0^t |F_k| \diff{s}, \\
    \left|\vw_k^{1}(t)-\tilde{\vw}_k^{1}(t)\right| &\le \int_0^t\left|a_k-\tilde{a}_k\right| \diff{s} +\int_0^t\left|G_k^1\right| \diff{s}.
\end{align*}
For $\alpha_1 > \frac{1}{2}$, it is clear that $|F_k| \le  C m q_{\rm{max}}^3$ and $|G_k^i| \le  C m q_{\rm{max}}^3$.
Recalling the definition of $r_{\rm{max}}$, it follows that $r_{\rm{max}}$ satisfies the following inequality
\begin{equation*}
    r_{\max }(t) \leq \int_0^t r_{\max }(s) \diff{s}+\int_0^t C m q_{\max }^3(s) \diff{s}.
\end{equation*}
Using explicit solution given by Eq.~\eqref{equ::analytical_sol} and initialization estimate from Eq.~\eqref{equ::initial_estimate}, we obtain the following for $t \le T_1$
\begin{align*}
    r_{\text {max}} & \le \int_0^t r_{\text {max}} \diff{s}+\int_0^t C m \frac{1}{m^{\alpha_1}} \log m q_{\text {max}}^2 \diff{s} \\
    & \leq \int_0^t r_{\text {max}} \diff{s}+\int_0^t C m \frac{1}{m^{\alpha_1}} \log m \left(\tilde{q}_{\text {max}}^2+r_{\text {max}}^2\right) \diff{s}  \\
    & \le \left(1+\frac{C m \log m }{m^{\alpha_1+\gamma_1}}\right) \int_0^t r_{\text {max}} \diff{s}+ \int_0^t C m \frac{1}{m^{\alpha_1}} \log m \left(\frac{1}{m^\alpha} \sqrt{2 \log \frac{2 m d}{\delta}} e^t\right)^2 \diff{s},
\end{align*}
where the first inequality follows from the definition of $T_1$, and the third inequality is derived from the explicit solution~\eqref{equ::analytical_sol}. Moreover, $\frac{C m \log m}{m^{\alpha_1 + \gamma_1}}\le 1$ for $m \ge C^{\frac{2}{2\alpha-1}}$. As a result, we have
\begin{equation*}
    r_{\rm{max}} \le 2 \int_0^t r_{\rm{max}} \diff{s} + \int_0^t C m \frac{1}{m^{\alpha_1}} \frac{1}{m^{2\alpha}} (\log m)^2 \E^{2t} \diff{s} .
\end{equation*}
By the Gronwall inequality, we have
\begin{equation}
    r_{\rm{max}} \le C m^{1- \alpha_1 - 2 \alpha} (\log m)^2 \E^{2t}.
\end{equation}

\item  Without loss of generality, we provide a detailed proof for $\frac{1}{2} < \alpha \le \frac{3}{2}$, as the proof for the other case is similar. We claim that $T_{\rm{p}} = \frac{2\alpha-1}{4} \log m \le T_1$. That is, the following inequality holds for $0<t\leq T_{\rm{p}}$
\begin{align*}
    q_{\rm{max}}(t) & \le \frac{1}{m^{\frac{1}{2} \alpha + \frac{1}{4}}} \log m, \\
    r_{\rm{max}}(t) & \le \frac{1}{m^{\frac{3}{2} \alpha - \frac{1}{4}}} (\log m)^3. 
\end{align*}
Now we prove by contradiction. If the assertion does not hold, then $T_1 < \frac{2\alpha-1}{4} \log m$. Hence
\begin{align*}
    r_{\rm{max}}(T_1) &\le C m^{1- \alpha_1 - 2 \alpha} (\log m)^2 \E^{2T_1} \\
    & < \frac{1}{m^{\alpha_1+\alpha-\frac{1}{2}}} (\log m)^3.
\end{align*}
Since $\gamma_1 = \alpha_1 + \alpha - \frac{1}{2}$, we have 
\begin{equation}
\label{equ::ini_rmax}
r_{\rm{max}} < \frac{1}{m^{\gamma_1}} (\log m)^3.    
\end{equation}
From the solution to the linearized dynamics~\eqref{equ::analytical_sol}, it follows that 
\begin{align}
\max_{k \in [m]} \{ a_k,w_k^1 \} 
     &\le \max_{k \in [m]} \{ \tilde{a_k},\tilde{w}_k^1 \} + r_{\rm{max}}\nonumber  \\
     &\le \frac{1}{m^\alpha} m^{\frac{2\alpha-1}{4}} \sqrt{2 \log \frac{2md}{\delta}} +\frac{1}{m^{\alpha_1+\alpha-\frac{1}{2}}} (\log m)^3\nonumber \\
     & < \frac{1}{m^{\frac{1}{4}+\frac{1}{2} \alpha }} \log m.\label{equ::init_a_k_w_k^1}
\end{align}
Since $\alpha_1 = \frac{1}{2} \alpha + \frac{1}{4}$
and for $w_k^i$ where $i \in [2:d]$ and $k \in [m]$, we apply Lemmas~\ref{lem::higher_term} and \ref{lem::growth} to obtain
\begin{align}
    \left| w_k^i (t)\right| &\le \left| w_k^i (0)\right| + \int_0^t  |a_k|\left( \sum_{l=1}^m |a_l| |w_l^i|\right) \diff{s} + \int_0^t |g_k^i| \diff{s}\nonumber \\
    & \le \left| w_k^i (0) \right| + \int_0^t  |a_k| \left(\sum_{l=1}^m a_l^2\right)^\frac{1}{2} \left(\sum_{l=1}^m (w_l^i)^2\right)^\frac{1}{2} \diff{s} + C (\log m)^4 \frac{1}{m^{\frac{3}{2}\alpha+ \frac{3}{4}}}\nonumber \\
    & \le \frac{1}{m^{\alpha}} \sqrt{2 \log \frac{2md}{\delta}} + C \frac{1}{m^{\frac{2 \alpha-1 }{2}}} \log m \frac{1}{m^{\frac{1}{2} \alpha + \frac{1}{4}}} \frac{1}{m^{\frac{1}{2} \alpha -\frac{1}{4}}}\nonumber\\
    & < \frac{1}{m^{\alpha}} \log m,\label{equ::init_w_k^i}
\end{align}
when $m$ is sufficiently large. This is just Eq.~\eqref{equ::initial_w_k^i} in Theorem \ref{thm::initial_stage}.  Consequently, from Eqs.~\eqref{equ::ini_rmax}, \eqref{equ::init_a_k_w_k^1}, and \eqref{equ::init_w_k^i}, we obtain 
\begin{align*} 
q_{\rm{max}} (T_1) &< \frac{1}{m^{\alpha_1}} \log m,\\
r_{\rm{max}} (T_1) &< \frac{1}{m^{\gamma_1}} (\log m)^3. 
\end{align*}
This contradicts the definition of $T_1$, thereby proving our claim.
\end{enumerate}
\end{proof}

Combined with some estimates on the initialization, we give the following fine estimate which will be used in the following sections.
\begin{prop}[Fine estimate at the initial plateau stage]
\label{prop::initial_global_estimate}
Let $\alpha_1, \gamma_1$, and $T_{\rm{p}}$ be as defined above. For $\delta \in (0,1)$ and sufficiently large $m$, with probability at least $1 - \delta$ over the choice of initial parameter $\vtheta(T_0)$, we have the following asymptotic equality at $t = T_{\rm{p}}$:
    \begin{equation} 
    \sum_{k=1}^m a_k w_k^1, \sum_{k=1}^m a_k^2 + (w_k^1)^2 = \Theta\left(\frac{1}{m^{2\alpha_1 -1}}\right).
    \end{equation}
Moreover,
    \begin{equation}
        \frac{\sum_{k=1}^m a_k w_k^1}{\sum_{k=1}^m a_k^2 + (w_k^1)^2} = \frac{1}{2} + \varepsilon_1,
    \end{equation}
    where $\varepsilon_1 = O(\frac{1}{m^{2\alpha_1 -1}} (\log m)^4)$.
\end{prop}

\begin{proof}
    We prove the case of $\frac{1}{2} < \alpha \le \frac{3}{2}$. The proof for the other case is similar. Based on the explicit solution~\eqref{equ::analytical_sol}, it follows that
    \begin{align*}
        \sum_{k=1}^m a_k w_k^1 
        &=  \sum_{k=1}^m \tilde{a}_k \tilde{w}_k^1 +\sum_{k=1}^m a_k w_k^1 -\sum_{k=1}^m \tilde{a}_k \tilde{w}_k^1 \\
        &= \sum_{k=1}^m \frac{1}{4} \E^{2t} \left(a_k (0) + w_k^1 (0) \right)^2 - \frac{1}{4} \E^{-2t}  \left(a_k (0) - w_k^1 (0) \right)^2 \\
        &\quad+ \sum_{k=1}^m a_k (w_k^1 -\tilde{w}_k^1) + \sum_{k=1}^m \tilde{w}_k^1 (a_k - \tilde{a}_k).
    \end{align*}
    According to Proposition~\ref{prop::ULB_of_ini_param}, which gives the initialization estimate, when $T_{\rm{p}} = \frac{2\alpha-1}{4} \log m$, the following inequalities hold at $t=0$ 
    \begin{align*}
        \frac{1}{8} \frac{1}{m^{\frac{2\alpha-1}{2}}} &\le  \sum_{k=1}^m \frac{1}{4} \E^{2t} \left(a_k (0) + w_k^1 (0) \right)^2 \le \frac{3}{8} \frac{1}{m^{\frac{2\alpha-1}{2}}}, \\
        \frac{1}{8} \frac{1}{m^{\frac{3(2\alpha-1)}{2}}} &\le  \sum_{k=1}^m \frac{1}{4} \E^{-2t} \left(a_k (0) - w_k^1 (0) \right)^2 \le \frac{3}{8} \frac{1}{m^{\frac{3(2\alpha-1)}{2}}}.
    \end{align*}
    Thanks to Theorem \ref{thm::initial_stage}, we have
    \begin{equation*}
        \left| \sum_{k=1}^m a_k (w_k^1 -\tilde{w}_k^1) \right|
        \le m q_{\rm{max}} r_{\rm{max}}
        \le \frac{1}{m^{2\alpha-1}} (\log m)^4
    \end{equation*}
    because $q_{\rm{max}}(t) \le \frac{1}{m^{\frac{1}{2} \alpha + \frac{1}{4}}} \log m$ and $r_{\rm{max}}(t) \le \frac{1}{m^{\frac{3}{2} \alpha - \frac{1}{4}}} (\log m)^3$. As a result, we have $\sum_{k=1}^m a_k w_k^1 = \Theta(\frac{1}{m^{\frac{2\alpha-1}{2}}}) = \Theta(\frac{1}{m^{2\alpha_1-1}})$. 
\end{proof}
Combining neuron-wise estimate and fine estimate at the initial plateau stage, we finish the proof of second part of Theorem \ref{thm::initial_stage} by direct calculation as follows.
\begin{proof}[Proof of Theorem~\ref{thm::initial_stage} (Part 2)]
    Using Taylor's expansion, we obtain that 
    \begin{equation}
        R (\vtheta) = \frac{1}{2} \int f^2(\vx) \rho(\vx) \diff{\vx} + \frac{1}{2} \sum_{i=1}^d \left(\sum_{k=1}^m a_k w_k^i \right)^2 - \sum_{k=1}^m a_k w_k^1 + \tilde{R}, 
    \end{equation}
    where 
\begin{align*}
     \tilde{R}  &= \int_{\sR^d} \left( \sum_{k=1}^m a_k \vw_k^{\T} \vx \right) \left( \frac{1}{3!} \sum_{k=1}^m \sigma^{(3)} (\xi_k(\vx)) a_k (\vw_k^{\T} \vx )^3 \right) \rho(\vx) \diff{\vx} \\ 
        &\quad + \frac{1}{2} \int_{\sR^d} \left( \frac{1}{3!} \sum_{k=1}^m \sigma^{(3)} (\xi_k(\vx)) a_k (\vw_k^{\T} \vx )^3 \right)^2 \rho(\vx) \diff{\vx} \\
        &\quad -\int_{\sR^d} \left( \frac{1}{3!} \sum_{k=1}^m \sigma^{(3)} (\xi_k(\vx)) a_k (\vw_k^{\T} \vx )^3 \right) f(\vx) \diff{\vx},
\end{align*}
Through direct calculation, the following inequality is valid when $0 \le K \le 2$ and $m q_{\rm{max}} \le 1$:
\begin{equation}
    \tilde{R} \le C m q_{\rm{max}}^4.
\end{equation}
We thus have
\begin{align}
    \left|R (\vtheta(T_{\rm{p}})) - R (\vtheta(0))\right| &\le C \left| K(T_{\rm{fp},\rm{\alpha}})-K(0)\right| + \sum_{i=2}^d \left( \sum_{k=1}^m (w_k^i (T_{\rm{fp},\rm{\alpha}}))^2 + (w_k^i (0))^2\right)\nonumber \\
    &\quad +  \left| \tilde{R}(T_{\rm{fp},\rm{\alpha}}) - \tilde{R} (0) \right|\nonumber \\
    &\le C  \frac{1}{m^{2\alpha_1 -1}} +C \max \left\{ \frac{1}{m^{2\alpha-1}}, \frac{(\log m)^{5}}{m^{5}} \right\} +  C m \frac{ (\log m)^{4}}{m^{4 \alpha_1}}\nonumber \\
    & \le C \frac{1}{m^{2\alpha_1 -1}}\label{equ::rela_change1}
\end{align}
for $\alpha_1 > \frac{1}{2}$ and $m$ is sufficiently large. By combining the above inequality with the definition of $T_{\rm{p}}$, we finish the proof.
\end{proof}

As a direct consequence of Theorem~\ref{thm::initial_stage} and Proposition~\ref{prop::initial_global_estimate}, we obtain a characterization of the Wasserstein distance and the phenomenon of initial condensation. The proof is omitted for brevity.

\begin{cor}[Amplitude distribution of weights are similar at initial plateau]
\label{cor::similar_fp}
    Let $\alpha_1, \gamma_1$, and $T_{\rm{p}}$ be as defined previously. For $\delta \in (0,1)$, with probability at least $1 - \delta$ over the choice of initial parameter $\vtheta(T_0)$, an upper bound for the relative Wasserstein distance at $t = T_{\rm{p}}$ is given by:
    \begin{equation}\label{eq...initial_wass}
        0 \le \frac{W_2 (\rho_{|a|}, \rho_{\norm{\vw}})}{W_2 (\rho_{|a|})} \le C \frac{1}{m^{\gamma_1 - \alpha_1}} (\log m)^{3} + C \max \left\{ \frac{1}{m^{\alpha-\alpha_1}}, \frac{(\log m)^{5}}{m^{3-\alpha_1}} \right\}.
    \end{equation}
\end{cor}

\begin{cor}[Condensation at initial plateau (Initial condensation)]
\label{cor::initial_condense}
    Let $\alpha_1, \gamma_1$, and $T_{\rm{p}}$ be as defined previously. For $\delta \in (0,1)$, with probability at least $1 - \delta$ over the choice of initial parameter $\vtheta(T_0)$, the following inequality holds at $t = T_{\rm{p}}$:
    \begin{equation}
    \label{eq...initial_condense}
        \frac{\sum_{k=1}^m \norm{\vw_k}_2^2}{\sum_{k=1}^m (w_k^1)^2} \geq 1 - C \max \left\{ \frac{1}{m^{2\alpha - 2\alpha_1}}, \frac{(\log m)^{10}}{m^{6 - 2\alpha_1}} \right\}.
    \end{equation}
\end{cor}

Given that $\alpha_1 < \min\{\gamma_1, \alpha, 3\}$, , it follows from inequality~\eqref{eq...initial_condense} that initial condensation occurs at $t = T_{\rm{p}}$, as $w_k^i$, for $i \in [2,d]$, remains in a lazy stage. Moreover, the distributions of $|a_k|$ and $\norm{\vw_k}$ converge to an identical form for sufficiently large $m$, as indicated by Eq.~\eqref{eq...initial_wass}.

\section{Initial Descent Stage}\label{sec...id}
In this section, we characterize the dynamics of the initial descent stage. We show that the behavior in this stage is primarily governed by Eq.~\eqref{equ::K_dynamics}, which guides the parameters toward convergence near a specified critical point. Specifically, we provide a theoretical estimate for the descent time and substantiate these findings through experimental validation.
\subsection{Descent time}
For brevity and readability, we focus on $\frac{1}{2} < \alpha \le \frac{3}{2}$ in this subsection, as the proof and results for $\alpha > \frac{3}{2}$ are analogous.
As discussed above, we aim to show that the dynamics will intuitively converge to a critical point corresponding to $K = 1$. Recall the definitions of $K$ and $K'$ are
\begin{align*}
    K & = \sum_{k=1}^m a_k w_k^1,  \\
    K'& = \sum_{k=1}^m  \left( a_k^2 + (w_k^1)^2 \right).
\end{align*}
For given small $\beta>0$, we define $T_{\rm{d}}$ to characterize the milestone for the initial descent stage,
\begin{equation}\label{defi...T_d}
    T_{\rm{d}} = \inf \left\{ T_{\rm{p}} \le t ~ | ~  K \leq 0, K'\leq 0, 1-K \leq \beta \right\}.
\end{equation}
The main theorem of this section can be stated as follows.
\begin{thm}[Initial descent stage]
\label{thm::descent}
For $\delta \in (0,1)$ and sufficiently large $m$, with probability at least $1 - \delta$ over the choice of initial parameter $\vtheta(T_0)$, for $\alpha > \frac{1}{2}$ and any given small $\beta, \delta' > 0$, the following holds:
\begin{enumerate}[label = (\roman*).,wide, labelwidth=!,labelindent=0pt]
\item 
    The asymptotic behavior of $T_{\rm{d}}$ is characterized by:
    \begin{equation}
    \label{eq:descent_time}
    \frac{2\alpha-1}{2+\delta'} \le \frac{T_{\rm{d}}}{\log m} \le \frac{2\alpha-1}{2-\delta'}
    \end{equation}
\item 
    The relative change in loss between the initial and descent stages differs significantly. Specifically, we have:
    \begin{equation}
      \frac{\left|R (\vtheta(T_{\rm{d}})) - R (\vtheta(T_{\rm{p}})) \right|}{\left| R (\vtheta(T_{\rm{p}})) - R (\vtheta(0)) \right|} \cdot \frac{\left|T_{\rm{p}} - T_0 \right|}{\left|T_{\rm{d}} - T_{\rm{p}}\right|} \ge C m^{2\alpha_1-1}.
    \end{equation}
\item 
    Furthermore, neuron-wise control holds at $t=T_{\rm{d}}$:
    \begin{align*}
    q_{\rm{max}} (T_{\rm{d}}) &\le \frac{1}{m^{\frac{1}{2}-(2\alpha-1)\delta'}}, \\
    \max_{k \in m, i \in [2:d]} \left|w_k^i (T_{\rm{d}})\right| &\le \frac{1}{m^{\alpha-2(2\alpha-1)\delta'}}.
    \end{align*}
\end{enumerate}
\end{thm}

\begin{rmk}
    Theorem~\ref{thm::descent} actually shows that two-layer neural network with different scales of parameters will converge to a ``critical point'' at fixed time after proper normalization.
\end{rmk}

To prove above result, we need some prior knowledge to control weights in non-condensed directions. We define 
\begin{equation}\label{defi...T_2}
    T_2 = \inf \{ t ~| ~ q_{\rm{max}} \ge \frac{1}{m^{\frac{1}{2} - \alpha_2}}\},
\end{equation}
where $\alpha_2 = (2\alpha-1) \delta'$. 
For sufficiently large $m$, we can choose $\delta'$ to be arbitrarily small. Consequently, we expect a critical bound for $q_{\rm{max}}$, as the term involving $\delta'$ approaches its limiting behavior, imposing a sharp constraint on the growth of $q_{\rm{max}}$.

Similar to Proposition~\ref{prop::initial_control}, we control $\sum_{k=1}^m (w_k^i)^2$ for $i=2,\dots,d$ by Lemma~\ref{lem::growth}. 

\begin{prop}[Estimate of weights in non-condensed directions at initial descent stage]
\label{prop::descent_control}
For $\delta \in (0,1)$ and sufficiently large $m$, with probability at least $1 - \delta$ over the choice of initial parameter $\vtheta(T_0)$, for $T_{\rm{p}} \le t \le \min \{ T_{\rm{d}}, T_2, \frac{2\alpha-1}{2-\delta'} \log m \}$ and $i=2,\dots,d$, $\sum_{k=1}^m (w_k^i)^2$ will be controlled by following inequality.
    \begin{equation}
        \left( \sum_{k=1}^m (w_k^i)^2  \right)^{\frac{1}{2}} \le C \max \left\{ \frac{1}{m^{\frac{2\alpha-1}{2}}} , \frac{\log m}{m^{1-5 \alpha_2}} \right\}. 
    \end{equation}
\end{prop}

\begin{proof}
    By growth Lemma~\ref{lem::growth}, 
    \begin{equation*}
        \left( \sum_{k=1}^m \left( w_k^i (t) \right)^2  \right)^{\frac{1}{2}} \le \left( \sum_{k=1}^m \left( w_k^i (T_{\rm{p}}) \right)^2 \right)^{\frac{1}{2}} + C \max \left\{ \frac{1}{m^{3(\frac{1}{2}-\alpha_2)-\frac{1}{2} }}, \frac{1}{m^{5(\frac{1}{2}-\alpha_2) -\frac{3}{2}}}\right\} (t-T_{\rm{p}}).  
    \end{equation*}
    By Proposition \ref{prop::initial_control} and the chosen of small $\delta'$
    \begin{equation*}
        \left( \sum_{k=1}^m \left( w_k^i (t) \right)^2 \right)^{\frac{1}{2}} \le  C \max \left\{ \frac{1}{m^{\frac{2\alpha-1}{2}}}, \frac{(\log m)^{5}}{m^{\frac{5}{2}}} \right\} + C \frac{1}{m^{1-5\alpha_2}} (t-T_{\rm{p}}).  
    \end{equation*}
As a result, for $t \le \min \{ T_{\rm{d}}, T_2, \frac{2\alpha-1}{2-\delta'} \log m \}$, we have
\begin{equation*}
    \left( \sum_{k=1}^m (w_k^i)^2 \right)^{\frac{1}{2}} \le C \max \left\{ \frac{1}{m^{\frac{2\alpha-1}{2}}} , \frac{\log m}{m^{1-5 \alpha_2}} \right\}.
\end{equation*}
\end{proof}

Using the estimate of weights in non-condensed directions, we derive the following monotonicity proposition. This proposition allows us to easily estimate the rate between $K$ and $K'$.
\begin{prop}[Monotonicity of $K$ and $K'$]
\label{prop::mono}
    For $T_{\rm{p}} \le t \le \min \{ T_{\rm{d}}, T_2, \frac{2\alpha-1}{2-\delta'} \log m \}$, $K$ and $K'$ increase monotonically.
\end{prop}
\begin{proof}
    We first examine the dynamics of $K$. According to the definition of $K$, we have 
    \begin{align*}
    \frac{\D K}{\D t} &= \sum_{k=1}^m \left(\frac{\D a_k}{\D t} w_k^1 + a_k \frac{\D w_k^1}{\D t} \right) \\
                      &= \sum_{k=1}^m \left( w_k^1 - \left( \sum_{l=1}^m a_l w_l^1 \right) w_k^1 - \sum_{i=2}^d \left( \sum_{l=1}^m a_l w_l^i \right) w_k^i + f_k \right) w_k^1 \\
                      &\quad + \sum_{k=1}^m \left( a_k - \left( \sum_{l=1}^m a_l w_l^1 \right) a_k +g_k^1\right) a_k.
    \end{align*}
    Recalling the definition of $K'$, we obtain
    \begin{align*}
        \frac{\D K}{\D t} = K'(1-K) - \sum_{i=2}^d \left( \sum_{k=1}^m a_k w_k^i \right) \left( \sum_{k=1}^m w_k^1 w_k^i \right) + \sum_{k=1}^m w_k^1 f_k + \sum_{k=1}^m a_k g_k^1.
    \end{align*}
     According to the Cauchy-Schwartz inequality, it is evident that
     \begin{align*}
          \left| \left( \sum_{k=1}^m a_k w_k^i \right) \left( \sum_{k=1}^m w_k^1 w_k^i \right) \right| & \le K' \left(\sum_{k=1}^m (w_k^i)^2 \right), \\
          \left| \sum_{k=1}^m w_k^1 f_k\right| + \left| \sum_{k=1}^m a_k g_k^1\right| & \le \sqrt{K'} \left( \sum_{k=1}^m f_k^2 \right)^{\frac{1}{2}}.
     \end{align*}
     Utilizing Proposition~\ref{prop::descent_control} and estimate of weights in non-condensed directions given in Eq.~\eqref{equ::higher-order_control}, we conclude
    \begin{align*}
        \left| \left( \sum_{k=1}^m a_k w_k^i \right) \left( \sum_{k=1}^m w_k^1 w_k^i \right) \right| &\le  C K' \max \left\{ \frac{1}{m^{2\alpha-1}} , \frac{(\log m)^2}{m^{2-10 \alpha_2}} \right\},\\
        \left| \sum_{k=1}^m w_k^1 f_k\right| + \left| \sum_{k=1}^m a_k g_k^1\right| &\le \sqrt{K'} \frac{1}{m^{1-5 \alpha_2}}.
    \end{align*}
    By setting $\varepsilon_2 = \Theta(\max \left\{\frac{1}{m^{\frac{1}{4}}} ,\frac{1}{m^{\frac{2\alpha-1}{2}}} \right\} )$, the following inequalities are satisfied at $T_{\rm{p}}$:
    \begin{align}
           \sqrt{K'} \frac{1}{m^{1-5 \alpha_2}} & \le \varepsilon_2  K' \beta,\nonumber \\
         \max \left\{ \frac{1}{m^{2\alpha-1}} , \frac{(\log m)^2}{m^{2-10 \alpha_2}} \right\} &\le \varepsilon_2 \beta.
    \label{equ::sqrtK'}
    \end{align}

    In Theorem~\ref{thm::initial_stage}, we select different values of $T_i$ for different $\alpha$due to the technical requirement that the above inequality holds. This is why there is no essential difference between the two cases. Consequently, $K$ will be monotonically increasing for at least a certain period. And we have 
    \begin{equation}
        (1-2\varepsilon_2) K'(1-K) \le \frac{\D K}{\D t} \le (1+ 2\varepsilon_2) K'(1-K) .
    \end{equation}
    Next, we investigate the dynamics of  $K'$. When Eq.~\eqref{equ::sqrtK'} is satisfied,  it follows that 
    \begin{align*}
        \frac{\D K'}{\D t} &= 4K(1-K) - 2 \sum_{i=2}^d \left( \sum_{k=1}^m a_k w_k^i \right)^2 + \sum_{k=1}^m a_k f_k + \sum_{k=1}^m w_k^1 g_k \\ 
        & \ge 4K(1-K) - 2\varepsilon_2 K'(1-K) \\
        &\ge (1-K) \left\{ 4K(T_{\rm{p}})- 2\varepsilon_2 K'(T_{\rm{p}}) + \int_{t_0}^t (4-C\varepsilon_2) K'(1-K) \diff{s}\right\}.  \\
    \end{align*}
    As a consequence, $K'$ will remain monotonically increasing for at least a certain duration, implying that Eq.~\eqref{equ::sqrtK'} will hold until $ \min \{ T_{\rm{d}}, T_2, \frac{2\alpha-1}{2-\delta'} \log m \}$, owing to the monotonicity of $K'$.
\end{proof}

Now we want to show $\frac{K}{K'} \approx \frac{1}{2}$ until $ \min \{ T_{\rm{d}}, T_2, \frac{2\alpha-1}{2-\delta'} \log m\}$ and then get a control of $K$. The main technique we use is the conservation law. First, we have the following proposition.

\begin{prop}[Approximate conservation law]
\label{prop::conservation_law}
    For $T_{\rm{p}} \le t \le \min \{ T_{\rm{d}}, T_2, \frac{2\alpha-1}{2-\delta'} \log m \}$, we have:
    \begin{equation}
        \left| 4K \frac{\D K}{\D t} - K' \frac{\D K'}{\D t} \right| \le C \varepsilon_2 (K')^2.
    \end{equation}
\end{prop}
\begin{proof}
    Analogous to the proof of Proposition \ref{prop::mono}, we derive 
    \begin{align*}
        K \frac{dK}{dt} &= KK'(1-K)- K \sum_{i=2}^d \left(\sum_{k=1}^m a_k w_k^i \right)  \left(\sum_{k=1}^m w_k^1 w_k^i \right) - K \left( \sum_{k=1}^m w_k^1 f_k + \sum_{k=1}^m a_k g_k \right),\\
        K' \frac{d K'}{d t}& =4KK'(1-K) - 2K' \sum_{i=2}^d \left( \sum_{k=1}^m a_k w_k^i \right)^2 - K'\left(\sum_{k=1}^m a_k f_k +  \sum_{k=1}^m w_k^1 g_k \right) .
    \end{align*}
    Hence, it is derived that
    \begin{align*}
        4K \frac{\D K}{\D t} - K' \frac{\D K'}{\D t} &= - 4K \sum_{i=2}^d \left(\sum_{k=1}^m a_k w_k^i \right)\left(\sum_{k=1}^m w_k^1 w_k^i \right) + 2K' \sum_{i=2}^d \left( \sum_{k=1}^m a_k w_k^i \right)^2 \\
        & \quad - 4K \left( \sum_{k=1}^m w_k^1 f_k + \sum_{k=1}^m a_k g_k \right) + K'\left(\sum_{k=1}^m a_k f_k +  \sum_{k=1}^m w_k^1 g_k \right).
    \end{align*}
    By the Cauchy-Schwartz inequality, if follows that 
    \begin{align*}
        \left|4K \frac{\D K}{\D t} - K' \frac{\D K'}{\D t} \right| &\le C (K')^2 \left( \sum_{i=2}^d \sum_{k=1}^m  (w_k^i)^2 \right) + C (K')^{\frac{3}{2}} \sum_{k=1}^m \left( f_k^2 + (g_k^1)^2 \right)^\frac{1}{2} \\ 
        & \le C \varepsilon_2 (K')^2.
    \end{align*}
\end{proof}

\begin{prop}[Approximate equality condition]
\label{prop::1/2}
    For $\delta \in (0,1)$ and sufficiently large $m$, with probability at least $1 - \delta$ over the choice of initial parameter $\vtheta(T_0)$, for $T_{\rm{p}} \le t \le \min \{ T_{\rm{d}}, T_2, \frac{2\alpha-1}{2-\delta'} \log m \}$, we have 
    \begin{equation}
        \frac{K}{K'} = \frac{1}{2} +\varepsilon_3,
    \end{equation}
    where $\varepsilon_3 = O \max \{ \varepsilon_1,  \varepsilon_2 \frac{2\alpha-1}{2} \log m\}$.
    Similarly, 
    \begin{equation}
        \frac{\sum_{k=1}^m a_k^2}{\sum_{k=1}^m (w_k^1)^2} = 1 + \varepsilon_3.
    \end{equation}
\end{prop}
\begin{proof}
According to direct calculation, we derive that
    \begin{align*}
        \frac{K}{K'} &= \sqrt{\frac{K^2 (T_{\rm{p}}) + K^2 (t) -K^2 (T_{\rm{p}})}{{K'}^2 (T_{\rm{p}}) + {K'}^2 (t) -{K'}^2 (T_{\rm{p}})}} \\
        & = \sqrt{\frac{(\frac{1}{2}+\varepsilon_1)^2{K'}^2 (T_{\rm{p}})+ \frac{1}{4} ({K'}^2 (t) -{K'}^2 (T_{\rm{p}})) + \int_{T_i}^t K \frac{\D K}{\D s} - \frac{1}{4} K' \frac{\D K'}{\D s} \diff{s}}{{K'}^2 (T_{\rm{p}}) + {K'}^2 (t) -{K'}^2 (T_{\rm{p}})}} \\
        & = \sqrt{\frac{1}{4} + (\varepsilon_1+\varepsilon_1^2) \frac{{K'}^2 (T_{\rm{p}})}{{K'}^2 (t)} +  \frac{\int_{T_i}^t K \frac{\D K}{\D s} - \frac{1}{4} K' \frac{\D K'}{\D s} \diff{s}}{{K'}^2 (t)}}.
    \end{align*}
Thanks to Propositions~\ref{prop::mono} and~\ref{prop::conservation_law}, the first equation is obtained through direct calculation.
    When $ \frac{1}{2} \le \frac{\sum_{k=1}^m a_k^2}{\sum_{k=1}^m (w_k^1)^2} \le 2$, $K'$ will satisfy 
    \begin{equation*}
        \frac{1}{C} \sum_{k=1}^m (w_k^1)^2\le K' \le C \sum_{k=1}^m (w_k^1)^2.
    \end{equation*}
    Then the second equation follows from above inequality and similar proof of first equality.
\end{proof}

We now proceed with the proof of the first part of Theorem~\ref{thm::descent}.
\begin{proof}[Proof of Theorem~\ref{thm::descent} (Part 1)]
This part of the proof can be divided into two steps. First, we claim that $\min\{ T_{\rm{d}}, T_2 \} \le \frac{2\alpha-1}{2-\delta'} \log m$. After that we prove $T_{\rm{d}} \le T_2$. The other side of the inequality~\eqref{eq:descent_time} is similar.
\begin{enumerate}[label = (\roman*).,wide, labelwidth=!,labelindent=0pt]
    \item From Eq.~\eqref{equ::K_dynamics},it follows that for $T_{\rm{p}} \le t \le \min \{ T_{\rm{d}}, T_2, \frac{2\alpha-1}{2-\delta'} \log m\} $, the inequality is satisfied
    \begin{equation*}
        (1-2\varepsilon_2) K'(1-K) \le \frac{\D K}{\D t} \le (1+ 2\varepsilon_2) K'(1-K) .
    \end{equation*}
    According to Proposition~\ref{prop::1/2}, it can be  derived that
    \begin{equation*}
         (2-\varepsilon_4) K(1-K) \le \frac{\D K}{\D t} \le (2+\varepsilon_4) K(1-K) ,
    \end{equation*}
    where $\varepsilon_4 = \min \{ \varepsilon_2, \varepsilon_3 \}$.
    
    As a result, when $T_{\rm{p}} \le t \le \min \{T_{\rm{d}}, T_2, \frac{2\alpha-1}{2-\delta'} \log m\}$, the following inequality holds:
    \begin{equation*}
    \frac{C \E^{(2-\varepsilon_4)(t-T_{\rm{p}})}}{1+C \E^{(2-\varepsilon_4)(t-T_{\rm{p}})}} \le K \le \frac{C \E^{(2+\varepsilon_4)(t-T_{\rm{p}})}}{1+C \E^{(2+\varepsilon_4)(t-T_{\rm{p}})}}, 
    \end{equation*}
    where $C = \Theta( \frac{1}{m^{\frac{2\alpha-1}{2}}} ) $ according to Proposition~\ref{prop::initial_global_estimate}. It is claimed that $\min \{T_{\rm{d}}, T_2 \} \le \frac{2\alpha-1}{2-\delta'} \log m$. If this is not the case, it follows that 
    \begin{align*}
    K   & \ge \frac{C \E^{(2-\varepsilon_4)(t-T_{\rm{p}})}}{1+C \E^{(2-\varepsilon_4)(t-T_{\rm{p}})}}\\
        & \ge \frac{\frac{C}{m^{\frac{2\alpha-1}{2}}} \E^{(2-\varepsilon_4)\frac{2\alpha-1}{4} (1+ \frac{2 \delta'}{2-\delta'}) \log m}}{1 + \frac{C}{m^{\frac{2\alpha-1}{2}}} \E^{ (2 - \varepsilon_4 )\frac{2\alpha-1}{4} (1+ \frac{2 \delta'}{2-\delta'}) \log m}} \\
        &\ge 1 - \frac{C}{m^{\frac{2\alpha-1}{4} (\delta'- 2 \varepsilon_4)}}. \\
    \end{align*}
    If $m$ is sufficiently large, it leads to a contradiction since $\frac{C}{m^{\frac{2\alpha-1}{4} (\delta'- 2 \varepsilon_4)}} < \beta$.

    \item Next, we assert that $T_{\rm{d}} \le \frac{2\alpha-1}{2-\delta'} \log m$. 
    Since it has been estabilished that $\min \{ T_{\rm{d}},T_2\} \le \frac{2\alpha-1}{2-\delta'} \log m$, if $T_{\rm{d}} \le T_2$, the proof is complete. Conversely, we proceed by contradiction. In this scenario, we have, $T_2 < T_{\rm{d}} \le \frac{2\alpha-1}{2-\delta'} \log m$.
    
    First, for $w_k^i$ where $i \in [2:d]$ and $k \in [m]$, we have 
    \begin{align*}
        \left| w_k^i (T_2)\right| &\le \left| w_k^i (T_{\rm{p}})\right| + \int_{T_{\rm{p}}}^{T_2}  |a_k| \left(\sum_{l=1}^m |a_l| |w_l^i|\right) \diff{s} + \int_{T_{\rm{p}}}^{T_2} |g_k^i| \diff{s}\\
        & \le \left| w_k^i (T_{\rm{p}}) \right| + \int_{T_{\rm{p}}}^{T_2}  |a_k| \left(\sum_{l=1}^m a_l^2\right)^\frac{1}{2} \left(\sum_{l=1}^m (w_l^i)^2\right)^\frac{1}{2} \diff{s} + \int_{T_{\rm{p}}}^{T_2} |g_k^i| \diff{s} \\
        & \le \frac{1}{m^{\alpha}} \log m + C \max \left\{ \frac{1}{m^{\frac{2\alpha-1}{2}}} , \frac{\log m}{m^{1-5 \alpha_2}} \right\}  \log m \frac{1}{m^{\frac{1}{2} - \alpha_2}}+ C \frac{1}{m^{\frac{3}{2}-5 \alpha_2}} \\
        & \le C \left( \frac{\log m}{m^{\alpha}} + \frac{\log m}{m^{\alpha-\alpha_2}} + \frac{(\log m)^2}{m^{\frac{3}{2}-6 \alpha_2}} \right).
    \end{align*}
Thus, when $t \le T_2$ ,  we have the following for sufficiently large $m$: 
\begin{equation}
    \max_{k \in [m], i \in [2:d]} |w_k^i| \le C\frac{\log m}{m^{\alpha-\alpha_2}} < \frac{1}{m^{\frac{1}{2}-\alpha_2}},
\end{equation}
which means $q_{\rm{max}}$ can only be chosen from $a_k,w_k^1$ at time $T_2$. In particular, we get 
\begin{equation*}
    \max_{k \in [m], i \in [2:d]} |w_k^i| \le \frac{1}{m^{\alpha-2\alpha_2}}.
\end{equation*}
Therefore, we only need to consider $a_k$ and $w_k^1$. We will demonstrate that for $T_{\rm{p}} \le t \le T_2$,
    \begin{align*}
        \frac{\D a_k}{\D t} &= w_k^1 - \left( \sum_{l=1}^m a_l w_l^1 \right) w_k^1 - \sum_{i=2}^d \left( \sum_{l=1}^m a_l w_l^i \right) w_k^i + f_k,\\
        \frac{\D w_k^1}{\D t} & = a_k - \left( \sum_{l=1}^m a_l w_l^1 \right) a_k +g_k^1. \\
    \end{align*}
Then, for $T_{\rm{p}} \le t \le T_2$, parameters $a_k$ and $w_k^1$ satisfy  the following inequalities:
\begin{align*}
    |a_k(t)| &\le |a_k(T_{\rm{p}})| + \int_{T_{\rm{p}}}^t |w_k^1| \diff{s} + \int_{T_{\rm{p}}}^t C \max \left\{ \frac{1}{m^{\frac{2\alpha-1}{2}}} , \frac{\log m}{m^{1-5 \alpha_2}} \right\} \frac{1}{m^{\alpha- 2\alpha_2}}+ \frac{1}{m^{\frac{3}{2}- 5 \alpha_2}} \diff{s}, \\
    |w_k^1(t)| & \le |w_k^1 (T_{\rm{p}})| + \int_{T_{\rm{p}}}^t |a_k| \diff{s} + \int_{T_{\rm{p}}}^t \frac{1}{m^{\frac{3}{2}- 5 \alpha_2}} \diff{s}.
\end{align*}
Let $v(t) = |a_k(t)| + |w_k^1 (t)|$. For $T_{\rm{p}} \le t \le T_2$, we have the following:
\begin{equation*}
    v(t) \le C\left( \frac{\log m}{m^{\frac{1}{2} \alpha + \frac{1}{4}}} +\frac{\log m}{m^{2\alpha - \frac{1}{2}-2 \alpha_2}} + \frac{(\log m)^2}{m^{1+\alpha-7 \alpha_2}} + \frac{\log m}{m^{\frac{3}{2}-5 \alpha_2}}\right) + \int_{T_{\rm{p}}}^t v(s) \diff{s} .
\end{equation*}
As a result, we will determine at $t = T_2$ using the Gronwall inequality that
    \begin{align*}
        v(T_2) & \le \E^{(T_2 - T_{\rm{p}})} C\left( \frac{\log m}{m^{\frac{1}{2} \alpha + \frac{1}{4}}} +\frac{\log m}{m^{2\alpha - \frac{1}{2}-2 \alpha_2}} + \frac{(\log m)^2}{m^{1+\alpha-7 \alpha_2}} + \frac{\log m}{m^{\frac{3}{2}-5 \alpha_2}}\right) \\
        & \le \E^{(\frac{2\alpha-1}{2-\delta'} \log m-T_{\rm{p}})} C\left( \frac{\log m}{m^{\frac{1}{2} \alpha + \frac{1}{4}}} +\frac{\log m}{m^{2\alpha - \frac{1}{2}-2 \alpha_2}} + \frac{(\log m)^2}{m^{1+\alpha-7 \alpha_2}} + \frac{\log m}{m^{\frac{3}{2}-5 \alpha_2}}\right) \\
        & \le m^{\frac{2\alpha-1}{4}} m^{(2\alpha-1) \frac{\delta'}{2(2-\delta')}} C\left( \frac{\log m}{m^{\frac{1}{2} \alpha + \frac{1}{4}}} +\frac{\log m}{m^{2\alpha - \frac{1}{2}-2 \alpha_2}} + \frac{(\log m)^2}{m^{1+\alpha-7 \alpha_2}} + \frac{\log m}{m^{\frac{3}{2}-5 \alpha_2}}\right) \\
        & \le C \frac{\log m}{m^{\frac{1}{2} -(2\alpha-1) \frac{\delta'}{2(2-\delta')}}} .
    \end{align*}
Since we take $\alpha_2$ to be  $(2\alpha-1) \delta'$, we can choose $m$ sufficiently large such that $q_{\rm{max}}(T_2) < \frac{1}{m^{\frac{1}{2}-\alpha_2}}$, which leads to a contradiction.

\item $T_{\rm{d}} \ge \frac{2\alpha-1}{2+\delta'} \log m$ is similar based on $T_{\rm{d}} \le T_2 \le \frac{2\alpha-1}{2-\delta'} \log m$ and the dynamics of $K$. 
\end{enumerate}
\end{proof}

We now prove the second part of Theorem~\ref{thm::descent} by direct calculation. 

\begin{proof}[Proof of Theorem~\ref{thm::descent} (Part 2)]
Based on the control of $\sum_{k=1}^m (w_k^i)^2$ and the definition of $T_{\rm{d}}$, we find that 
\begin{align}
    \left|R (\vtheta(T_{\rm{d}})) - R (\vtheta(T_{\rm{p}}))\right| 
    &\ge \left| \frac{1}{2} K^2 (T_{\rm{d}}) - K(T_{\rm{d}})\right| -  \sum_{i=2}^d \left( \sum_{k=1}^m (w_k^i (T_{\rm{d}}))^2 + (w_k^i (T_{\rm{fp},\rm{\alpha}}))^2\right)\nonumber \\
    &\quad - \left| \tilde{R}(T_{\rm{fp},\rm{\alpha}}) - \tilde{R} (0) \right|\nonumber \\
    &\ge \frac{1}{2} (1-\beta) -   C \max \left\{ \frac{1}{m^{2\alpha-1}} , \frac{(\log m)^2}{m^{2-10 \alpha_2}} \right\} - C m \frac{1}{m^{2-4\alpha_2}}\nonumber \\
    &\ge \frac{1}{4}.\label{equ::rela_change2}
\end{align}
Combining Eqs.~\eqref{equ::rela_change1} and \eqref{equ::rela_change2}, we finish the proof. 
\end{proof}

As mentioned above, during the initial descent stage, the neural network converges to a point close to a critical point. However, due to higher-order terms, the network does not converge to an exact critical point. Thus, the concept of a critical point must be appropriately relaxed. Based on this, we define a relative critical point and we find that neural network will converge to a $\beta$-relative critical point at $T_{\rm{d}}$.
\begin{cor}[$\beta$-relative critical point convergence]
    For $\delta \in (0,1)$ and sufficiently large $m$, with probability at least $1 - \delta$ over the choice of initial parameter $\vtheta(T_0)$, the original dynamics~\eqref{equ::whole_dynamics} converge to a $\beta$-relative critical point at $ t = T_{\rm{d}}$, that is
    \begin{equation}
        \norm{\nabla R(\vtheta)}_{\infty} \le \beta\norm{\vtheta}_{\infty}.
    \end{equation}
    On the other hand, at $t = T_{\rm{fp},\rm{\alpha}}$, we have 
    \begin{equation}
        \norm{\nabla R(\vtheta)}_{\infty} \ge \left(1-\frac{C}{m^{2\alpha_1}}\right) \norm{\vtheta}_{\infty} 
    \end{equation}
\end{cor}
\begin{proof}
    The proof of both inequalities is based on direct calculation and Theorems~\ref{thm::initial_stage},\ref{thm::descent}. 
\end{proof}

The following provide a broad interpretation of the Wasserstein distance.
\begin{cor}[Amplitude distribution of weights are similar at initial descent stage] 
\label{cor::descent_distribution}
For $\delta \in (0,1)$ and sufficiently large $m$, with probability at least $1 - \delta$ over the choice of initial parameter $\vtheta(T_0)$, the distribution $\rho_{|a|}$ and $\rho_{\norm{\vw}}$ are similar at $T_{\rm{d}}$. That is at $t=T_{\rm{d}}$, 
    \begin{equation}
    0 \le \frac{W_2 (\rho_{|a|}, \rho_{\norm{\vw}})}{W_2 (\rho_{|a|})} \le \frac{1}{m^{\alpha- \frac{1}{2}}} +  C \max \left\{ \frac{1}{m^{\frac{2\alpha-1}{2}}} , \frac{\log m}{m^{1-5 \alpha_2}} \right\}.
    \end{equation}
\end{cor}

\begin{rmk}
    Corollaries~\ref{cor::initial_condense}, \ref{cor::descent_distribution} together imply that the distributions $\rho_{|a|}$ and $\rho_{\norm{\vw}}$ will remain approximately similar for an extended period. 
\end{rmk}

\begin{proof}
    By Theorem~\ref{thm::initial_stage}, at $T_{\rm{p}}$
    \begin{equation*}
        r_{\rm{max}} \le \frac{(\log m)^{\beta_2}}{m^{\gamma_1}}.
    \end{equation*}
    By direct calculation, we get 
    \begin{align*}
        \left|a_k(t)- w_k^1 (t)\right| &\le  \left|a_k(T_{\rm{p}})- w_k^1 (T_{\rm{p}})\right| + \int_{T_{\rm{p}}}^t \left|a_k(s)- w_k^1 (s)\right| \diff{s} \\
        &~~ +\int_{T_{\rm{p}}}^t C \max \left\{ \frac{1}{m^{\frac{2\alpha-1}{2}}} , \frac{\log m}{m^{1-5 \alpha_2}} \right\} \frac{1}{m^{\alpha- 2\alpha_2}} + \frac{1}{m^{\frac{3}{2} - 5\alpha_2}} \diff{s}.
    \end{align*}
    By Gronwall inequality, we have
    \begin{align*}
        r_{\rm{max}} (T_{\rm{d}}) &\le  C \E^{T_{\rm{d}}-T_{\rm{p}}} \max \left\{\frac{(\log m)^{\beta_2}}{m^{\gamma_1}}, \frac{\log m}{m^{2\alpha -\frac{1}{2} -2\alpha_2}}, \frac{\log m}{m^{1+\alpha-7\alpha_2}}, \frac{\log m}{m^{\frac{3}{2}-5\alpha_2}} \right\} \\
        &\le C \max \left\{ \frac{(\log m)^{\beta_2}}{m^{\alpha- \alpha_2}},\frac{\log m}{m^{\frac{3}{2} \alpha - \frac{1}{4}- 3\alpha_2}}\right\}.
    \end{align*}
    Since $\alpha_2$ can be sufficiently small when $m$ is large, we find that 
    \begin{equation*}
        r_{\rm{max}} (T_{\rm{d}}) \le \frac{1}{m^{\frac{1}{2} \alpha + \frac{1}{4}}},
    \end{equation*}
    which means 
    \begin{equation}
        \sum_{k=1}^m \left( |a_k|- |w_k^1|\right)^2 \le  \sum_{k=1}^m (a_k - w_k^1)^2 \le \frac{1}{m^{\alpha- \frac{1}{2}}} .
    \end{equation}
    The results follows from above inequality since $\alpha > \frac{1}{2}$ and $\sum_{k=1}^m a_k^2 = \Theta(1)$.
\end{proof}

\subsection{Experiment}\label{sec...exp_descent_time}
In this subsection, we conduct experimental analyses to examine the relationship between descent time $T_{\rm{d}}$ and various hyperparameters. The initial descent milestone, $T_{\rm{d}}$, is defined as the endpoint of the initial descent stage in the loss curve. Beyond $T_{\rm{d}}$, the loss transitions into the secondary plateau stage. In Theorem~\ref{thm::descent}, we establish a result for $T_{\rm{d}}$ in terms of the initial standard deviation $\alpha$ and the network width $m$. Recall the relation~\eqref{eq...descent_time_informal}:
\begin{equation*}
   \lim_{m \rightarrow \infty} \frac{ T_{\rm{d}}}{\log m} = \frac{2\alpha-1}{2}.
\end{equation*}
This expression highlights the asymptotic relationship between $T_{\rm{d}}$, $\alpha$, and $m$ as the width grows.
We conduct several experiments for different values of $\alpha$ and $m$, identifying the initial descent milestone and plotting the results in Fig.~\ref{fig:descent_time}(a) confirms that $T_{\rm{d}}$ exhibits a linear relationship with $\alpha$ for sufficient large $m$. Similarly, Fig.~\ref{fig:descent_time}(b) illustrates that $T_{\rm{d}}$ is linearly related to $\log m$. This verifies our estimate even for not sufficiently large $m$, e.g. $m$ = 1000. The details of the experiment is provided in Appendix~\ref{app...exp}. 

\begin{figure}[h]
    \centering
    \subfigure[]{
    \includegraphics[width=0.4\linewidth]{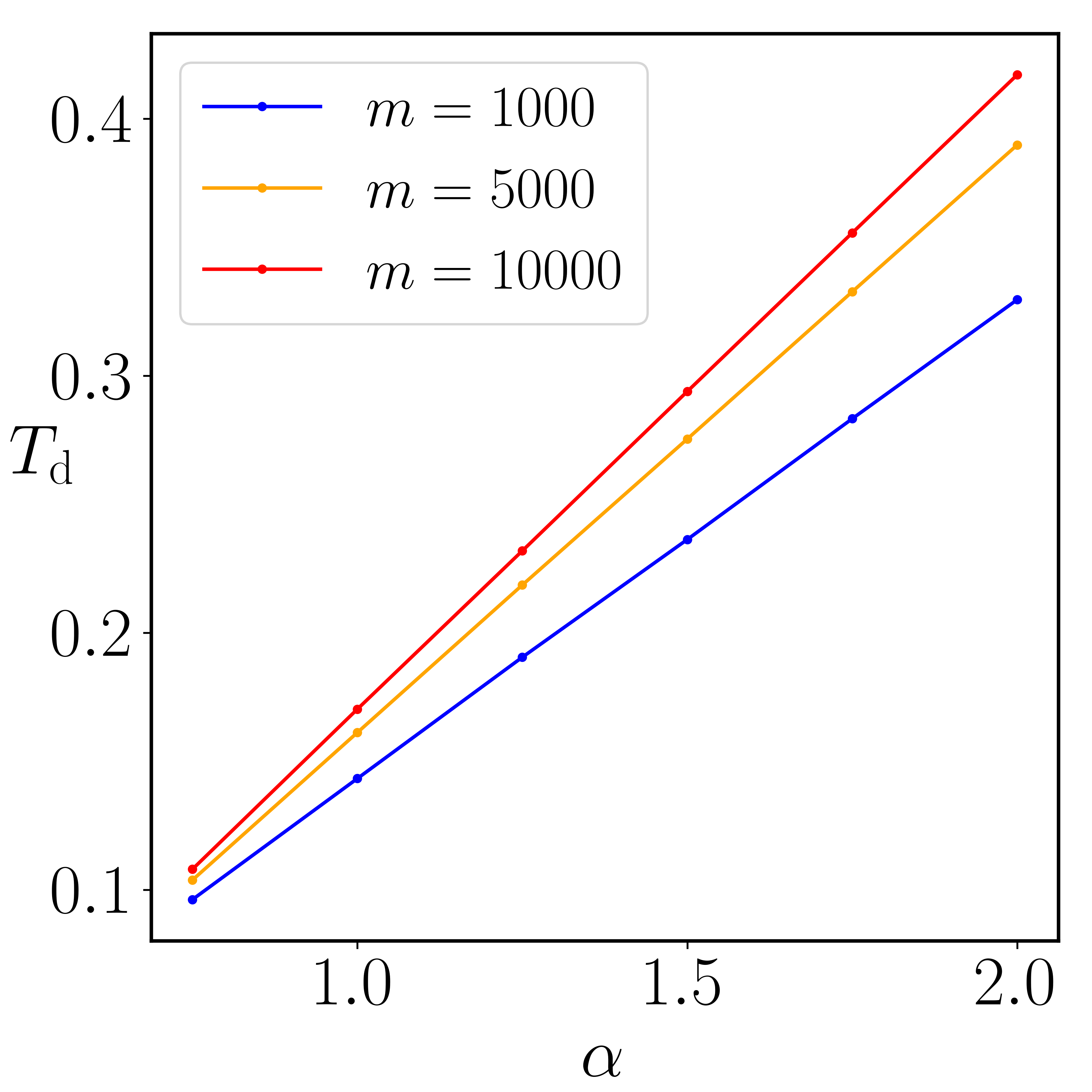}
    }
    \subfigure[]{
    \includegraphics[width=0.4\linewidth]{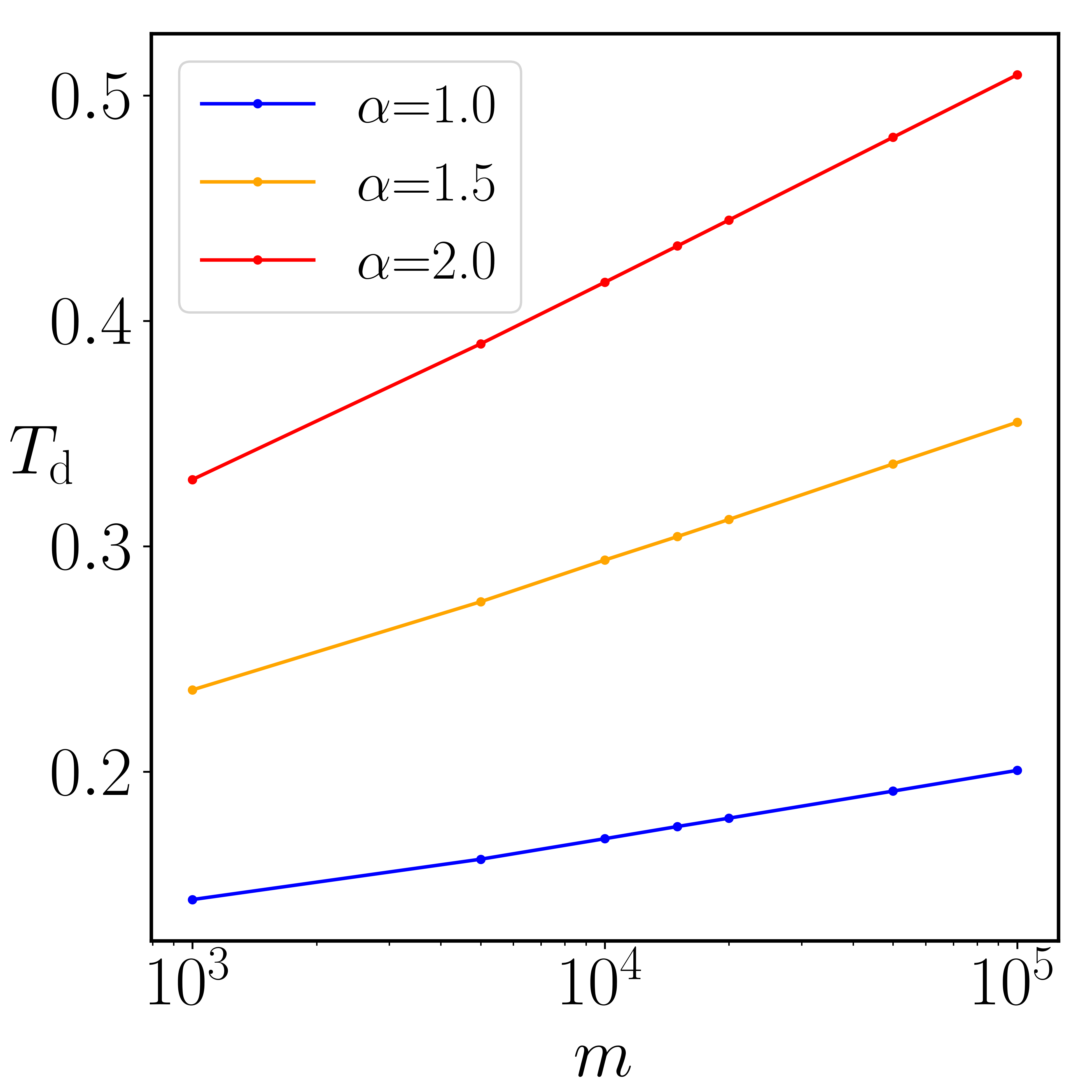}
    }
    \caption{Descent time $T_{\rm{d}}$ with respect to $\alpha$ (panel (a)) and $m$ (panel (b))}
    \label{fig:descent_time}
\end{figure}

\section{Secondary Plateau Stage}\label{sec...sp}
In this section, we explore the dynamic properties of the secondary plateau stage. Through rigorous theoretical analysis, we show that this stage lasts considerably longer than the preceding stages. Additionally, heuristic analysis suggests that, after a prolonged period, the parameters gradually depart from this plateau. We validate our findings experimentally using the Wasserstein distance.
\subsection{Long plateau}\label{sec...longtermplateau}
We now analyze the dynamics behavior at the secondary plateau stage and show that the dynamics can be approximately viewed as Eq.~\eqref{equ::sec_dynamics} due to the stability of the dynamics near critical point.  Also, we consider the case $\frac{1}{2} < \alpha \le \frac{3}{2}$ and another case is similar. First, we recall for sufficiently large $m$ and small $\delta'$, the following inequalities hold at $t = T_{\rm{d}}$ by Theorem~\ref{thm::descent}:
\begin{enumerate}[label=(\roman*).] 
    \setlength{\itemsep}{0.5em}
    \item $|K-1| \le \delta$ and $|K'-2| \le \delta$.
    \item $q_{\rm{max}} \le \frac{1}{m^{\frac{1}{2}-(2\alpha-1)\delta'}}$.
    \item $\max_{k \in m, i \in [2:d]} |w_k^i| \le \frac{1}{m^{\alpha-2(2\alpha-1)\delta'}}$
    \item $\left( \sum_{k=1}^m (w_k^i)^2 \right)^{\frac{1}{2}} \le C \min \left\{ \frac{1}{m^{\frac{2\alpha-1}{2}}} , \frac{\log m}{m^{1-5(2\alpha-1)\delta'}} \right\}$
\end{enumerate}

Based on above estimation, we give a description of the secondary plateau stage. 
Define $T_3 = \sup \{t~|~ q_{\rm{max}} \le \frac{1}{m^{\alpha_3}} \}$ where $\alpha_3 < \frac{1}{2}-\alpha_2$ and  $T_4 = \sup \{ t~|~ 1 \le K' \le 3 \}$, where $C>1$. Here we take $\alpha_3 = 0.4$ since $\delta'$ can be sufficiently small if $m$ is large. Note that setting $\alpha_3 =0.4$ is simply to demonstrate that the neural network remains near the ``critical point'' for an extended period. We also define
\begin{equation}\label{defi...T_sp}
T_{\rm{sp}} = T_{\rm{d}}+ \frac{1}{40} \frac{1}{\beta} \log m    
\end{equation}
as the end of the secondary plateau stage. This indicates that the secondary plateau will persist for a significant period.

\begin{thm}[Long secondary plateau]
\label{thm::second_plateau}
    For $\delta \in (0,1)$, sufficiently large $m$ and given $\alpha$, with probability at least $1 - \delta$ over the choice of initial parameter $\vtheta(T_0)$, the secondary plateau will longer than the first two stages significantly. That is 
    \begin{equation}
    \label{equ::long_sp}
        \frac{T_{\rm{sp}}-T_{\rm{d}}}{T_{\rm{d}}} \gg 1.
    \end{equation}
    Moreover, loss will will stagnate for an extended period on the secondary plateau. That is 
    \begin{equation}
    \label{equ::sp_small_change}
        \frac{ \left| R (\vtheta(T_{\rm{sp}})) - R (\vtheta(T_{\rm{d}})) \right|}{\left| R (\vtheta(T_{\rm{d}})- R (\vtheta(T_{\rm{p}})) \right|} \frac{T_{\rm{d}}- T_{\rm{p}}}{T_{\rm{sp}}-T_{\rm{d}}} \approx 0.
    \end{equation}
\end{thm}

The overall proof is similar to the previous one. Based on growth Lemma~\ref{lem::growth}, we have following proposition. 
\begin{prop}[Estimate of weights in non-condensed directions at secondary plateau stage]
 For $\delta \in (0,1)$ and sufficiently large $m$, with probability at least $1 - \delta$ over the choice of initial parameter $\vtheta(T_0)$, for $T_{\rm{d}} \le t \le \min \{ T_3, T_4, T_{\rm{sp}} \}$ and $i=2,\dots,d$, $\sum_{k=1}^m (w_k^i)^2$ will be controlled by following inequality,
\begin{equation}
     \left( \sum_{k=1}^m \left( w_k^i (t) \right)^2  \right)^{\frac{1}{2}} \le C \max \left\{ \frac{1}{m^{\frac{2\alpha-1}{2}}} , \frac{\log m}{m^{1-5(2\alpha-1)\delta'}} \right\} + C \frac{(\log m)^2}{m^{5 \alpha_3 - \frac{3}{2}}}.
\end{equation}
\end{prop}

\begin{proof}
By growth Lemma~\ref{lem::growth}, we get for sufficiently large $m$
    \begin{align*}
        \left( \sum_{k=1}^m \left( w_k^i (t) \right)^2  \right)^{\frac{1}{2}} &\le \left( \sum_{k=1}^m \left( w_k^i (T_{\rm{d}}) \right)^2 \right)^{\frac{1}{2}} \\
        &\le C \max \left\{ \frac{1}{m^{\frac{2\alpha-1}{2}}} , \frac{\log m}{m^{1-5(2\alpha-1)\delta'}} \right\} + C \frac{1}{\beta} \frac{\log m}{m^{5 \alpha_3 - \frac{3}{2}}} \\
        &\le C \max \left\{ \frac{1}{m^{\frac{2\alpha-1}{2}}} , \frac{\log m}{m^{1-5(2\alpha-1)\delta'}} \right\} + C \frac{(\log m)^2}{m^{5 \alpha_3 - \frac{3}{2}}}.
    \end{align*}
\end{proof}
Now we turn to the proof of Theorem~\ref{thm::second_plateau}.
\begin{proof}[Proof of Theorem~\ref{thm::second_plateau}]
    By definition of $T_{\rm{sp}}$, inequality~\eqref{equ::long_sp} is straightforward. Therefore, what actually needs to be proven is that the loss is indeed at the secondary plateau at time $T_{\rm{sp}}$, which corresponds to the inequality~\eqref{equ::sp_small_change}. With the help of $T_3$, this only requires proving 
    \begin{equation}
        T_{\rm{sp}} \le T_3.
    \end{equation}
    The remaining part only requires simple calculations. The whole proof can be divided into three parts.
\begin{enumerate}[label = (\roman*).,wide, labelwidth=!,labelindent=0pt]
\item We first establish a control of $w_k^i$ where $i\in [2:d]$ at $t = \min \{ T_3, T_4, T_{\rm{sp}}\}$. 
That is 
\begin{align*}
    |w_k^i (t)| &\le |w_k^i (T_{\rm{d}})| + \int_{T_{\rm{d}}}^t |a_k| \sum_{k=1}^m |a_k| |w_k^i| \diff{s} + \int_{T_{\rm{d}}}^t |g_k^i| \diff{s} \\
    & \le \frac{1}{m^{\alpha-2(2\alpha-1)\delta'}} + \int_{T_{\rm{d}}}^t \frac{1}{m^{\alpha_3}} C \max \left\{ \frac{1}{m^{\frac{2\alpha-1}{2}}} , \frac{\log m}{m^{1-5(2\alpha-1)\delta'}} \right\} \diff{s}\\
    &+ \int_{T_{\rm{d}}}^t  C \frac{\log m}{m^{5 \alpha_3 - \frac{3}{2}}} \diff{s}.
\end{align*}
As a result, $|w_k^i| < \frac{1}{m^{\alpha_3}}$ at $t = \min \{ T_3, T_4, \frac{1}{40} \frac{1}{\beta} \log m\}$, given that $\alpha_3 = 0.4$ and  $\beta$ is a given small constant, which allows us to choose $m$ to be sufficiently large. This implies that $q_{\rm{max}}$ can't be attained by $w_k^i$.

\item Next, we examine the dynamics of $K$: 
    \begin{align*}
        \frac{\D K}{\D t} = K'(1-K) - \sum_{i=2}^d \left( \sum_{k=1}^m a_k w_k^i \right) \left( \sum_{k=1}^m w_k^1 w_k^i \right) +  \sum_{k=1}^m w_k^1 f_k + \sum_{k=1}^m a_k g_k^1.
    \end{align*}
Then for $t \le \min \{ T_3, T_4, T_{\rm{sp}}\}$, we have 
\begin{equation*}
    (K-1)^2 (t) - (K-1)^2 (T_{\rm{d}}) \le C \varepsilon \int_{T_{\rm{d}}}^t |K-1| \diff{s}.
\end{equation*}
where $\varepsilon = \min \left\{ \frac{1}{m^{2\alpha-1}} , \frac{(\log m)^2}{m^{2-10(2\alpha-1)\delta'}},\frac{1}{m^{5 \alpha_3 - \frac{3}{2}}} \right\}$. By applying the Gronwall inequality, we find that the following inequality holds for sufficiently large $m$,
\begin{equation}
    |K-1|(t) \le \beta + C \varepsilon \frac{1}{\beta} \log m \le 2\beta.
\end{equation}

\item We assert that $\min \left\{ T_3,T_4 \right\}\ge \frac{1}{40} \frac{1}{\beta} \log m$. 
If this is not the case, there are two possible cases:
\begin{itemize}
    \item $T_3 \le \frac{1}{40} \frac{1}{\beta} \log m$ and $T_3 \le T_4$. 
    \item $T_4 \le \frac{1}{40} \frac{1}{\beta} \log m$ and $T_4 \le T_4$.
\end{itemize}
We will proceed with a case-by-case analysis. If the first case holds, the main technique we use is the conservation law. Since
\begin{align*}
    4K \frac{\D K}{\D t} - K' \frac{\D K'}{\D t} &= - 4K \sum_{i=2}^d \left(\sum_{k=1}^m a_k w_k^i \right)\left(\sum_{k=1}^m w_k^1 w_k^i \right) + 2K' \sum_{i=2}^d \left( \sum_{k=1}^m a_k w_k^i \right)^2 \\
    &- 4K \left( \sum_{k=1}^m w_k^1 f_k + \sum_{k=1}^m a_k g_k \right) + K'\left(\sum_{k=1}^m a_k f_k +  \sum_{k=1}^m w_k^1 g_k \right).
\end{align*}
The right hand of the equality is less than $C \varepsilon (K')^2$. Similar to the proof of Proposition~\ref{prop::1/2}, $K/K' \approx \frac{1}{2}$, which means
$K' \approx 2$ since $|K-1| \le 2 \beta$. It contradicts with the definition of $T_3$.

If the second case holds, the technique we now use is the point estimation which is similar to what we have done before but with some difference.
\begin{align*}
    \frac{\D a_k}{\D t} &= w_k^1 - \left( \sum_{l=1}^m a_l w_l^1 \right) w_k^1 - \sum_{i=2}^d \left( \sum_{l=1}^m a_l w_l^i \right) w_k^i + f_k,\\
    \frac{\D w_k^1}{\D t} & = a_k - \left( \sum_{l=1}^m a_l w_l^1 \right) a_k +g_k^1.
\end{align*}
Then for $t \le T_4$, we obtain
\begin{align*}
    |a_k(t)| &\le |a_k(T_{\rm{p}})| + 2\beta \int_{T_{\rm{p}}}^t |w_k^1| \diff{s} + \int_{T_{\rm{p}}}^t C \min\{ \frac{1}{m^{\frac{2\alpha-1}{2}}}, \frac{1}{m^{1-5c_0}}\} \frac{1}{m^{\alpha- c_0}} \diff{s} ,\\
    |w_k^1(t)| & \le |w_k^1 (T_{\rm{p}})| + 2\beta \int_{T_{\rm{p}}}^t |a_k| \diff{s} + \int_{T_{\rm{p}}}^t \frac{1}{m^{\frac{3}{2}- 5c_0}} \diff{s}.
\end{align*}
This means that 
\begin{equation}
    v(t) \le (1+\E^{2\beta (t-T_{\rm{p}})} )v(t_0).
\end{equation}
Taking $t-T_{\rm{p}} = \frac{1}{40}\frac{1}{\beta} \log m$, we find that $q_{\rm{max}} \le \frac{1}{m^{0.45}}$. It leads to a contradiction.

As a result, $\min \left\{T_4,T_{\beta} \right\} \le \frac{1}{40} \frac{1}{\beta} \log m$. Moreover, we have $T_4 \le T_3$ by similar discussion.

\end{enumerate}
\end{proof}

\subsection{Departure from the plateau}\label{sec...plateau_departure}
The next question is how the dynamics drive the parameters away from the critical point. Here, we provide insights through experiments and heuristic analysis. Recall the definition of $f_k$ and $\vg_k$.
\begin{align*}
    f_k & = \int_{\sR^d} f(\vx) \left( \frac{1}{3!} \sigma^{(3)} (\xi_k(\vx)) (\vw_k^{\T} \vx)^3 \right) \rho(\vx) \diff{\vx} \\
    &\quad - \int_{\sR^d} \left( \sum_{l=1}^m \frac{1}{3!} \sigma^{(3)} (\xi_l(\vx)) a_l (\vw_l^{\T} \vx)^3  \right) (\vw_k^{\T} \vx) \rho(\vx) \diff{\vx} \\
    & \quad - \int_{\sR^d} \left( \sum_{l=1}^m a_l \left( \vw_l^{\T} \vx + \frac{1}{3!} \sigma^{(3)}(\xi_l(\vx)) (\vw_l^{\T} \vx)^3\right) \right) \left( \frac{1}{3!} \sigma^{(3)} (\xi_k(\vx)) (\vw_k^{\T} \vx)^3 \right)  \rho(\vx) \diff{\vx}.\\
    \vg_k &=\int_{\sR^d} f(\vx) a_k\left(\frac{1}{2!} \sigma^{(3)}\left(\eta_k (\vx)\right)\left(\vw_k^{\T} \vx\right)^2\right) \vx \rho(\vx) \diff{\vx}  \\
    & \quad -\int_{\sR^d} \left(\sum_{l=1}^m a_l \left(\vw_l^{\T} \vx\right)\right) a_k \left( \frac{1}{2!} \sigma^{(3)}\left(\eta_k (\vx)\right)\left( \vw_k^{\T} \vx \right)^2 \right) \vx \rho(\vx) \diff{\vx} \\
    & \quad -\int_{\sR^d} \left(\sum_{l=1}^m a_l \frac{1}{3!} \sigma^{(3)}\left(\xi_l (\vx)\right)\left(\vw_l^{\T} \vx\right)^3\right) a_k\left(1+\frac{1}{2!} \sigma^{(3)}\left(\eta_k (\vx)\right)\left(\vw_k^{\T} \vx \right)^2\right) \vx \rho(\vx) \diff{\vx}.
\end{align*}
In above discussion, we actually use linear model $\sum_{k=1}^m a_k \vw_k^{\T} \vx$ to characterize the behavior of the nonlinear model $f_{\vtheta}$. On the secondary plateau, the original system's nonlinearity gradually begins to emerge. The first noticeable effect is that, at this stage, the roles of $a$ and $\vw$ are no longer equivalent. For simplicity, we consider a $1$-d example for illustration and truncate only up to the third order, as we are considering the tanh activation function. As a result, $f_k$ and $\vg_k$ will be approximated by 
\begin{align*}
    f_k &= -\frac{1}{3!} c_4 \sigma^{(3)}(0) \left( \sum_{l=1}^m a_l w_l^3 \right) w_k + \frac{1}{3!} b_3 \sigma^{(3)}(0) w_k^3, \\
    \vg_k&= -\frac{1}{3!} c_4 \sigma^{(3)}(0) \left( \sum_{l=1}^m a_l w_l^3 \right) a_k + \frac{1}{2!} b_3 \sigma^{(3)}(0) a_k w_k^2 \\
    & \quad - \frac{1}{2!} \sigma^{(3)} (0) c_4 \left( \sum_{l=1}^m a_l w_l \right) a_k w_k^2, 
\end{align*}
where $c_4 = \int x^4 \rho(x) \D x$ and $b_3 = \int f(x) x^3 \rho(x) \D x$.
Substituting the condition of critical point \(\sum_{k=1}^m a_k w_k = 1\) into whole dynamics, we obtain \(\frac{\D(\sum_{k=1}^m a_k^2)}{\D(\sum_{k=1}^m w_k^2)} \neq 1\). 

Thus, the roles of $\va$ and $\vw$ begin to shift. Once this shift progresses sufficiently, the nonlinearity starts to fully emerge, causing the loss to depart from the secondary plateau.
\subsection{Experiment}\label{sec...exp_target}
In the previous proofs and discussions, we characterized the microscopic behavior of the parameters across the three stages. During the initial descent stage, the distributions of the input weights $\va$ and output weights $\vw$ gradually become consistent. At this point, training enters the secondary plateau. The descent of the training loss resumes only after the weight distributions diverge. Meanwhile, the weight norms exceed the linear regime, entering the nonlinear region. To experimentally demonstrate that this dynamical phenomenon exists universally under our assumptions, we test different target functions and display the ratio of their weights norms, along with the relative Wasserstein distance between the weights distributions in Fig.~\ref{fig:target_plot}. We can observe in Fig.~\ref{fig:target_plot}(a) that this ratio stays around the value of 1 (gold dash line) for a period before the time it begins to change, 
after which the behavior of weights depends on the settings of different target functions. Moreover, the relative Wassertein distances in Fig.~\ref{fig:target_plot} approach zero simultaneously, indicating that the input weights and output weight have the same distribution as well as the same norm at this time. From Fig.~\ref{fig:target_plot} we can obtain a consistent performance of the first three stages for different target functions, validating the correctness of our theorems. Additionally, the details of the setting are provided in the Appendix~\ref{app...exp}.
\begin{figure}[h]
    \centering
    \subfigure[${\norm{\va}_2}/{\norm{\vw}_{\rm{F}}}$]{
    \includegraphics[width=0.4\linewidth]{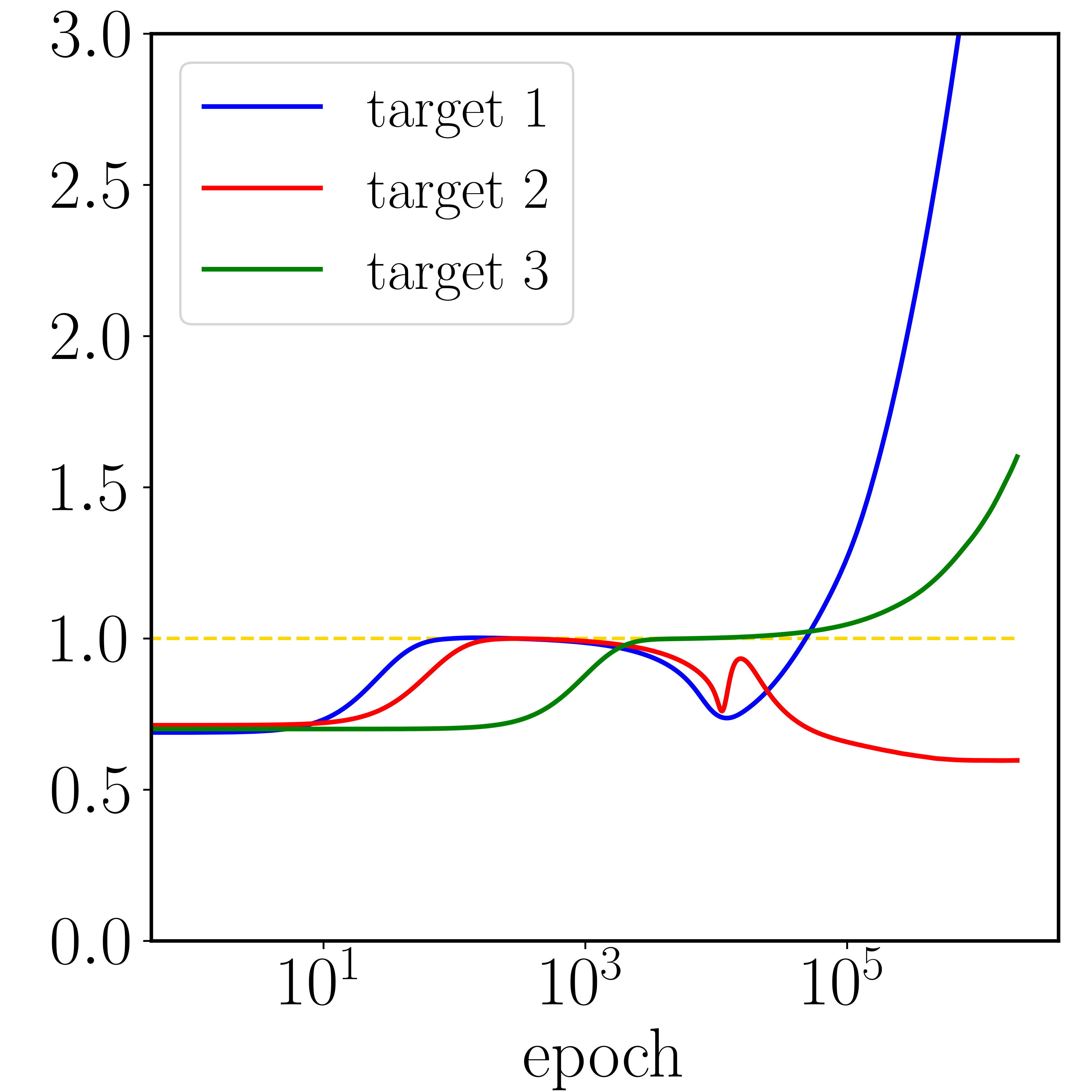}
    }
    \subfigure[$W^{\rm{rel}}_2 (\rho_{|a|}, \rho_{\| \vw \|})$]{
    \includegraphics[width=0.4\linewidth]{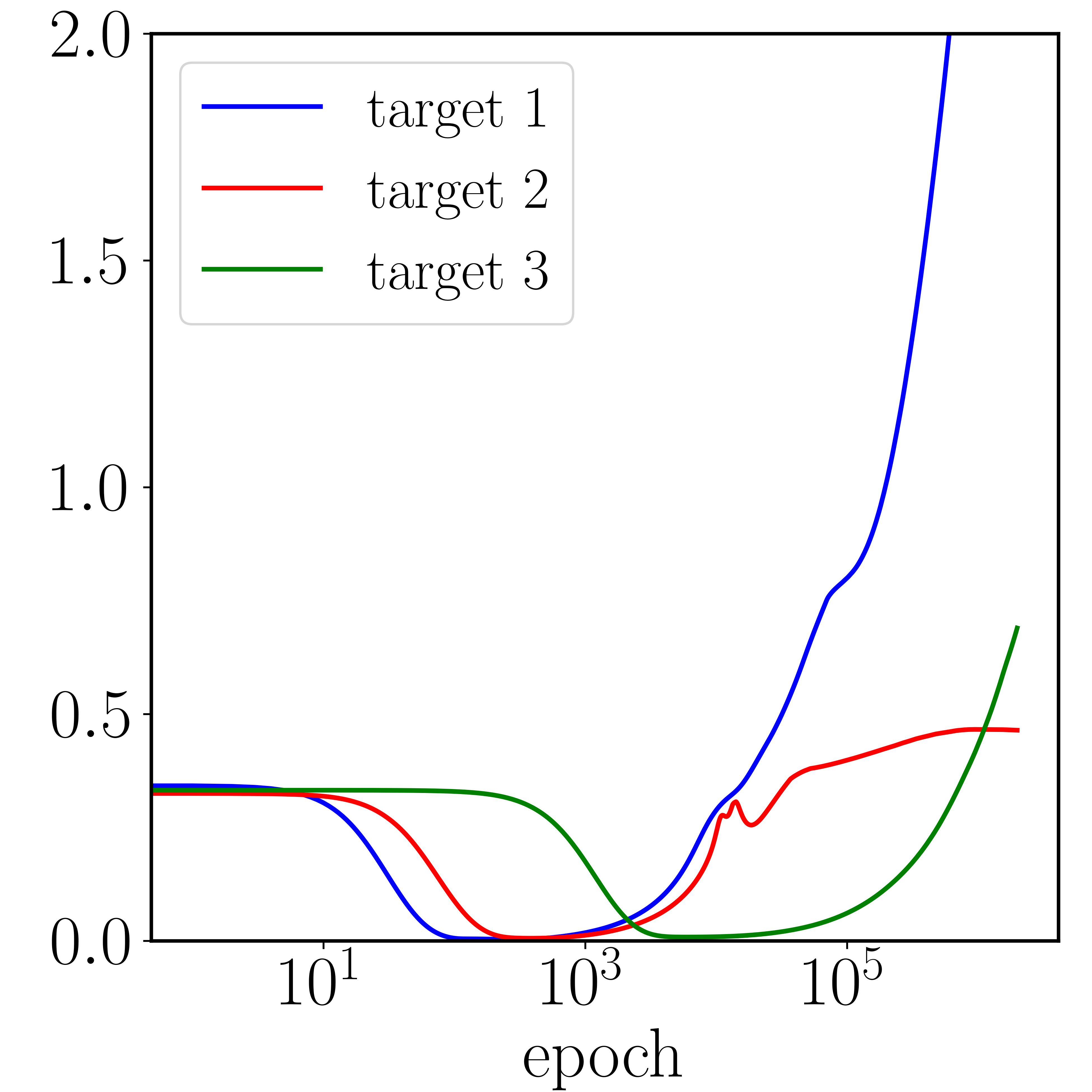}
    }
    \caption{Panel (a) and (b): Towards different target function, the curve of ratio of weight norms and the relative Wasserstein distance, respectively.}
    \label{fig:target_plot}
\end{figure}

\section{Summary}\label{sec...summary}
In this paper, we consider the dynamics of two-layer neural networks during training with small initialization. We identify three distinct stages and provide rigorous descriptions for each stage in an over-parameterized setting. We use linearized dynamics to characterize the original full dynamics and prove that parameters converge along a few direction in initial plateau stage. By considering higher-order terms, we propose an approximate system which can be analytically studied to describe the initial descent stage and verify the result experimentally. For secondary plateau stage, we theoretically characterize the duration of secondary plateau and give an intuitive explanation of the end of the plateau. Additionally, we introduce Wasserstein metric to characterize the distribution of weights during training. Our results further extend previous studies on the training process with small initialization. In particular, we provide new understanding of the training dynamics of neural networks, not only from a macroscopic perspective of loss performance but also from the more microscopic parameter behavior.

\acks{
This work is sponsored by the National Key R\&D Program of China Grant No. 2022YFA1008200 (T. L.), the National Natural Science Foundation of China Grant No. 12101401 (T. L.),  Shanghai Municipal Science and Technology Key Project No. 22JC1401500 (T. L.), Shanghai Municipal of Science and Technology Major Project No. 2021SHZDZX0102 (T. L.).}

\newpage
\appendix
\section{Notation List}\label{app...notation}
See Tab.~\ref{tab:notation} for key notation.
\begin{table}[ht]
    \centering
        \begin{tabular}{>{\centering\arraybackslash}m{0.2\linewidth}|>{\centering\arraybackslash}m{0.3\linewidth}||>{\centering\arraybackslash}m{0.2\linewidth}|>{\centering\arraybackslash}m{0.3\linewidth}}
        \toprule
        \textbf{Notation} & \textbf{Description}&\textbf{Notation} & \textbf{Description} \\
        \hline
        $m,\alpha$ & width and hyperparameter in initialization & $T_1,T_2,T_3,T_4$ & milestones for characterizing the stages, see Eqs.~\eqref{defi...T_1} and \eqref{defi...T_2} and Section~\ref{sec...longtermplateau} \\ 
        \hline
          $K,K^\prime$ & macroscopic quantities of parameters, see Eqs.~\eqref{defi::K} and ~\eqref{defi::K'} & $T_{\rm{test}}$ & intuitive conjecture milestone, see Lemma~\ref{lem::growth} \\
        \hline
        $q_{\rm{max}}$, $\tilde{q}_{\rm{max}}$ & neuron-wise $L_{\infty}$ norm, see Eqs.~\eqref{defi::q_max} and \eqref{defi...q_max_tilde} & $T_{\rm{p}}$ & milestone for the initial plateau stage, see Theorem~\ref{thm::initial_stage} \\
        \hline
        $r_{\rm{max}}$ & weights deviation from linearized dynamics, see Eq.~\eqref{defi...rmax} & $T_{\rm{d}}$ & milestone for the initial descent stage, see Eq.~\eqref{defi...T_d} \\
        \hline
        $W_2(\rho_1,\rho_2)$ & Wasserstein distance between $\rho_1$ and $\rho_2$, see Eq.~\eqref{defi....wassersteindist} &  $T_{\rm{sp}}$ & milestone for the secondary plateau stage, see Eq.~\eqref{defi...T_sp}  \\ 
        \hline

        \bottomrule
        \end{tabular}
    \caption{Notation list}
    \label{tab:notation}
\end{table}

\section{Derivation of Main Dynamical System}
In this section, we derive Eq.~\eqref{equ::grad_descent}. Since we consider the gradient flow, the dynamics of parameters reads as
\begin{align*}
    \frac{\D a_k}{\D t} &= - \nabla_{a_k}R, \\
    \frac{\D \vw_k}{\D t} &= - \nabla_{\vw_k}R.
\end{align*}
We sequentially calculate the equations for $a_k$ and $\vw_k$ as follows. First, for $a_k$, we have
\begin{align*}
    \frac{\D a_k}{\D t} &= - \int_{\sR^d} \left( f_{\vtheta}(\vx) - f(\vx)\right) \sigma(\vw_k^{\T} \vx) \rho(\vx) \diff{\vx}  \\
    &= -\int_{\sR^d} \left( \sum_{k=1}^m a_k \sigma(\vw_k^{\T} \vx) - f(\vx)\right) \sigma(\vw_k^{\T} \vx) \rho(\vx) \diff{\vx}.
\end{align*}
By Taylor's expansion and Assumption \ref{assump..TanhActivation}, we have 
\begin{equation*}
    \frac{\D a_k}{\D t} = - \int_{\sR^d} \left( \sum_{l=1}^m a_l \left( \vw_l^{\T} \vx + \frac{1}{3!} \sigma^{(3)}(\xi_l(\vx)) (\vw_l^{\T} \vx)^3\right) - f(\vx)\right) \left( \vw_k^{\T} \vx + \frac{1}{3!} \sigma^{(3)}(\xi_k(\vx)) (\vw_k^{\T} \vx)^3\right) \rho(\vx) \diff{\vx}.
\end{equation*}
Expanding and rearranging the terms, we obtain
\begin{align}
    \frac{\D a_k}{\D t} 
    &= \int_{\sR^d} f(\vx) (\vw_k^{\T} \vx) \rho(\vx) \diff{\vx} - \int_{\sR^d} \left( \sum_{l=1}^m  a_l \vw_l^{\T} \vx \right) (\vw_k^{\T} \vx) \rho(\vx) \diff{\vx}\nonumber \\
    &\quad + \int_{\sR^d} f(\vx) \left( \frac{1}{3!} \sigma^{(3)} (\xi_k(\vx)) (\vw_k^{\T} \vx)^3 \right) \rho(\vx) \diff{\vx}\nonumber \\
    &\quad - \int_{\sR^d} \left( \sum_{l=1}^m \frac{1}{3!} \sigma^{(3)} (\xi_l(\vx)) a_l (\vw_l^{\T} \vx)^3  \right) (\vw_k^{\T} \vx) \rho(\vx) \diff{\vx}\nonumber \\
    & \quad - \int_{\sR^d} \left( \sum_{l=1}^m a_l \left( \vw_l^{\T} \vx + \frac{1}{3!} \sigma^{(3)}(\xi_l(\vx)) (\vw_l^{\T} \vx)^3\right) \right) \left( \frac{1}{3!} \sigma^{(3)} (\xi_k(\vx)) (\vw_k^{\T} \vx)^3 \right)  \rho(\vx) \diff{\vx}.\label{equ::derivation_ak_1}
\end{align}
Note that Assumption \ref{assump..Non-degenDirect} leads to
\begin{equation}
\label{equ::derivation_ak_2}
    \int_{\sR^d} f(\vx) (\vw_k^{\T} \vx) \rho(\vx) \diff{\vx} =  w_k^1.
\end{equation}
Using Assumption \ref{assump..SymSamp}, we have 
\begin{align}
    \int_{\sR^d} \left( \sum_{l=1}^m  a_l \vw_l^{\T} \vx \right) (\vw_k^{\T} \vx) \rho(\vx) \diff{\vx} &= \int_{\sR^d} \left( \sum_{l=1}^m a_l \left( \sum_{i=1}^d w_l^i x_i\right)\right) \left( \sum_{j=1}^d w_k^j x_j \right) \rho(\vx) \diff{\vx}\nonumber \\
    &= \int_{\sR^d} \left( \sum_{i=1}^d \left(\sum_{l=1}^m a_l w_l^i \right) x_i\right) \left( \sum_{j=1}^d w_k^j x_j \right) \rho(\vx) \diff{\vx}\nonumber \\
    &= \sum_{i=1}^d \left( \sum_{l=1}^m a_l w_l^i\right) w_k^i.\label{equ::derivation_ak_3}
\end{align}
Combining Eqs.~\eqref{equ::derivation_ak_1}, \eqref{equ::derivation_ak_2}, \eqref{equ::derivation_ak_3} with the definition of $f_k$, we have 
\begin{equation*}
    \frac{\D a_k}{\D t} = w_k^1 - \sum_{i=1}^d \left( \sum_{l=1}^m a_l w_l^i\right) w_k^i + f_k.
\end{equation*}

Next, for $\vw_k$, we have 
\begin{align*}
    \frac{\D \vw_k}{\D t} &= - \int_{\sR^d} \left( f_{\vtheta} (\vx) - f(\vx) \right) a_k \sigma^{(1)} (\vw_k^{\T} \vx) \vx \rho(\vx) \diff{\vx} \\
    &= - \int_{\sR^d} \left( \sum_{l=1}^m a_l \sigma(\vw_l^{\T} \vx) -f(\vx) \right) a_k \sigma^{(1)} (\vw_k^{\T} \vx) \vx \rho(\vx) \diff{\vx}.
\end{align*}
By Taylor's expansion and Assumption~\ref{assump..TanhActivation}, we have
\begin{equation*}
    \frac{\D \vw_k}{\D t} = - \int_{\sR^d} \left( \sum_{l=1}^m a_l \left( \vw_l^{\T} \vx + \frac{1}{3!} \sigma^{(3)}(\xi_l(\vx)) (\vw_l^{\T} \vx)^3\right) - f(\vx)\right) a_k \left(1 + \frac{1}{2!} \sigma^{(3)} (\eta_k (\vx)) (\vw_k^{\T} \vx)^2 \right) \vx \rho(\vx) \diff{\vx}
\end{equation*}
Similarly, by properly reorganizing terms, the above expression can be written as:
\begin{align*}
    \frac{\D \vw}{\D t} &= - \int_{\sR^d} \left( \sum_{l=1}^m a_l (\vw_l^{\T} \vx) - f(\vx) \right) a_k \vx \rho(\vx) \diff{\vx} \\
    &\quad +\int_{\sR^d} f(\vx) a_k\left(\frac{1}{2!} \sigma^{(3)}\left(\eta_k (\vx)\right)\left(\vw_k^{\T} \vx\right)^2\right) \vx \rho(\vx) \diff{\vx}  \\
    & \quad -\int_{\sR^d}\left(\sum_{l=1}^m a_l \left(\vw_l^{\T} \vx\right)\right) a_k \left( \frac{1}{2!} \sigma^{(3)}\left(\eta_k (\vx)\right)\left( \vw_k^{\T} \vx \right)^2 \right) \vx \rho(\vx) \diff{\vx} \\
    & \quad -\int_{\sR^d}\left(\sum_{l=1}^m a_l \frac{1}{3!} \sigma^{(3)}\left(\xi_l (\vx)\right)\left(\vw_l^{\T} \vx\right)^3\right) a_k\left(1+\frac{1}{2!} \sigma^{(3)}\left(\eta_k (\vx)\right)\left(\vw_k^{\T} \vx \right)^2\right) \vx \rho(\vx) \diff{\vx}.
\end{align*}
Based on the definition of $\vg_k$, we have Eq.~\eqref{equ::grad_descent}. Writing this into the entry-wise form, we obtain Eq.~\eqref{equ::whole_dynamics}.

\section{Estimates on the Initialization}\label{secA1}
We begin this part by an estimate on standard Gaussian vectors.

\begin{lem}[Bounds on initial parameters; rephrased from Lemma 1 of \cite{chen2023phase}]\label{lem...init_est}
     Given any $\delta \in(0,1)$, we have with probability at least $1-\delta$ over the choice of $\vtheta$,
\begin{equation}
\label{equ::initial_estimate}
\max _{k \in[m]}\left\{|a_k(0)|,\norm{\vw_k(0)}_{\infty}\right\} \leq \frac{1}{m^{\alpha}}  \sqrt{2 \log \frac{2 m(d+1)}{\delta}}.    
\end{equation}

\end{lem}

\begin{proof}
    If $\rm{X} \sim N(0,1)$, then for any $\eta>0$, we have
$$
\Prob(|\rm{X}|>\eta) \leq 2 \exp \left(-\frac{1}{2} \eta^2\right).
$$
Since for $k \in[m]$,
$
a_k \sim N(0,\frac{1}{m^{2 \alpha}})$ and $\vw_k^0 \sim N\left(\mathbf{0}, \frac{1}{m^{2\alpha}}\mI_d\right),
$ with
$
\vw_k:=\left[w_k^1, w_k^2, \cdots, w_k^d \right]^{\T}
$, we have
then for any $k \in [m]$ and $j \in[d]$,
\begin{align*}
    m^{\alpha} a_k &\sim N(0,1), \\
    m^{\alpha} w_k^j &\sim N(0,1),
\end{align*}
and they are all independent with each other. Setting
$
\eta=\sqrt{2 \log \frac{2 m(d+1)}{\delta}},
$
we obtain
\begin{align*}
    & \Prob\left(\max _{k \in[m]}\left\{\left| m^{\alpha} a_k\right|,\left\| m^{\alpha} \vw_k\right\|_{\infty}\right\}>\eta\right) \\
    & = \Prob\left(\max _{k \in[m], j \in[d]}\left\{\left| m^{\alpha} a_k \right|,\left|m^{\alpha} w_k^j \right|\right\}>\eta\right) \\
    & = \Prob\left(\bigcup_{k=1}^m\left[\left(\left|m^{\alpha} a_k\right|>\eta\right) \bigcup\left(\bigcup_{j=1}^d\left(\left| m^{\alpha} w_k^j \right|>\eta\right)\right)\right]\right) \\
    & \leq \sum_{k=1}^m \Prob\left(\left| m^{\alpha} a_k\right|>\eta\right)+\sum_{k=1}^m \sum_{j=1}^d \Prob\left(\left|m^{\alpha} w_k^j \right|>\eta\right) \\
    & \leq 2 m \exp \left(-\frac{1}{2} \eta^2\right)+2 m d \exp \left(-\frac{1}{2} \eta^2\right) \\
    &= 2 m(d+1) \exp \left(-\frac{1}{2} \eta^2\right)=\delta.
\end{align*}
\end{proof} 

Next we introduce the sub-exponential norm (See, for example,~\cite{vershynin2010introduction}) of a random variable and Bernstein’s Inequality.
\begin{defi}[Sub-exponential norm]
The sub-exponential norm of a random variable $X$ is defined as
\begin{equation}
    \|\mathrm{X}\|_{\psi_1}:=\inf \left\{s>0 \mid \mathbb{E}_{\mathrm{X}}\left[\mathrm{e}^{[\mathrm{X} \mid / s}\right] \leq 2\right\}.
\end{equation}
    In particular, we denote the sub-exponential norm of a $\chi^2(d)$ random variable X by $C_{\psi, d}:=$ $\|\mathrm{X}\|_{\psi_1}$. Here the $\chi^2$ distribution with $d$ degrees of freedom has the probability density function
$$
f_{\mathrm{X}}(z)=\frac{1}{2^{d / 2} \Gamma(d / 2)} z^{d / 2-1} \mathrm{e}^{-z / 2}
$$

\end{defi}
\begin{rmk}
    Note that
\begin{align*}
    \mathbb{E}_{\mathrm{X} \sim \chi^2(d)} \mathrm{e}^{|\mathrm{X}| / 2} & =\int_0^{+\infty} \frac{1}{2^{d / 2} \Gamma(d / 2)} z^{d / 2-1} \mathrm{~d} z=+\infty, \\
    \lim _{s \rightarrow+\infty} \mathbb{E}_{\mathrm{X} \sim \chi^2(d)} \mathrm{e}^{|\mathrm{X}| / s} & =\lim _{s \rightarrow+\infty} \int_0^{+\infty} \frac{1}{2^{d / 2} \Gamma(d / 2)} z^{d / 2-1} \mathrm{e}^{-z / 2+z / s} \mathrm{~d} z=1<2.
\end{align*}

These imply that $2 \leq C_{\psi, d}<+\infty$.
\end{rmk}


\begin{thm}[Bernstein's inequality]
\label{thm::Bernstein}
     Let $\left\{\rm{X}_k\right\}_{k=1}^m$ be i.i.d. sub-exponential random variables satisfying
$$
\Exp X_1=\mu,
$$
then for any $\eta \geq 0$, we have
$$
\Prob\left(\left|\frac{1}{m} \sum_{k=1}^m \rm{X}_k-\mu\right| \geq \eta\right) \leq 2 \exp \left(-C_0 m \min \left(\frac{\eta^2}{\left\|\rm{X}_1\right\|_{\psi_1}^2}, \frac{\eta}{\left\|\rm{X}_1\right\|_{\psi_1}}\right)\right)
$$
for some absolute constant $C_0$.
\end{thm} 

\begin{prop}[Upper and lower bounds of initial parameters; rephrased from Proposition 1 of \cite{chen2023phase}]
\label{prop::ULB_of_ini_param}
Given any $\delta \in$ $(0,1)$, if
$
m=\Omega\left(\log \frac{4}{\delta}\right),
$
then with probability at least $1-\delta$ over the choice of $\vtheta$ we have
\begin{align*}
    \sqrt{\frac{(d+1)}{2} \frac{1}{m^{2\alpha-1}}} &\leq\left\|\vtheta\right\|_2 \leq \sqrt{\frac{3(d+1)}{2} \frac{1}{m^{2\alpha-1}}}, \\
    \sqrt{\frac{1}{2} \frac{1}{m^{2\alpha-1}}} &\leq\left\|\va\right\|_2 \leq \sqrt{\frac{3}{2}\frac{1}{m^{2\alpha-1}}}, \\
    \sqrt{\frac{d}{2} \frac{1}{m^{2\alpha-1}}} &\leq\left\|\vw\right\|_{\rm{F}} \leq \sqrt{\frac{3d}{2} \frac{1}{m^{2\alpha-1}}}.
\end{align*}
\end{prop}

\begin{proof}
Since
$$
\left( m^{\alpha} a_1\right)^2,\left(m^{\alpha} a_2\right)^2, \cdots,\left(m^{\alpha} a_m\right)^2 \sim \chi^2(1)
$$
are i.i.d. sub-exponential random variables with
$$
\Exp\left(m^{\alpha} a_1\right)^2=1.
$$
By application of Theorem \ref{thm::Bernstein}, we have
$$
\Prob\left(\left|\frac{1}{m} \sum_{k=1}^m\left( m^{\alpha} a_k\right)^2-1\right| \geq \eta\right) \leq 2 \exp \left(-C_0 m \min \left(\frac{\eta^2}{C_{\psi, 1}^2}, \frac{\eta}{C_{\psi, 1}}\right)\right),
$$
since $C_{\psi, 1} \geq \frac{8}{3}>2$, then for any $0 \leq \eta \leq 2$, it is obvious that
$$
\min \left(\frac{\eta^2}{C_{\psi, 1}^2}, \frac{\eta}{C_{\psi, 1}}\right)=\frac{\eta^2}{C_{\psi, 1}^2}.
$$
Setting
$$
2 \exp \left(-C_0 m \frac{\eta^2}{C_{\psi, 1}^2}\right)=\frac{\delta}{2},
$$
we have
$$
\eta=\sqrt{\frac{C_{\psi, 1}^2}{ m C_0} \log \frac{4}{\delta}}.
$$
Thus with probability at least $1-\frac{\delta}{2}$ over the choice of $\vtheta$,
$$
\frac{1}{m^{2\alpha-1}} \left(1-\sqrt{\frac{C_{\psi, 1}^2}{ m C_0} \log \frac{4}{\delta}}\right) \leq\left\|\va\right\|_2^2 \leq \frac{1}{m^{2\alpha-1}}\left(1+\sqrt{\frac{C_{\psi, 1}^2}{ m C_0} \log \frac{4}{\delta}}\right)
$$
By choosing
$$
m \geq \frac{4 C_{\psi, 1}^2}{C_0} \log \frac{4}{\delta},
$$
we obtain
$$
\sqrt{\frac{1}{m^{2\alpha-1}}\frac{1}{2}} \leq\left\|\va\right\|_2 \leq \sqrt{\frac{1}{m^{2\alpha-1}}\frac{3}{2}}.
$$
The other inequalities are similar.

\end{proof}

\section{Details of Experiment Settings}\label{app...exp}
Since the focus of this study is on the dynamical behavior during the training process of neural networks, it should be noted that we employ the gradient descent method with learning rate of $0.001$ for simulating the gradient flow, this will be consistently applied throughout our paper. More detailed settings are as followings:
\begin{enumerate}[label=\textbullet]
    \itemsep 0.5em
    \item $m = [1000, 5000, 10000, 15000, 20000, 50000, 100000]$,
    \item $\alpha = [0.75, 1.0, 1.25, 1.5, 1.75, 2.0]$,
    \item input data: $\{(x_i,f^*(x_i))\}_{i=1}^{n=1000}$, $\{x_i\}_i^{n}$ are  equidistant points in $[-15, 15]$,
    \item the target functions are $f^*_1 = \tanh(x+7.5) + \tanh(x) + \tanh(x-7.5)$; $f^*_2 = \tanh(x)$; $f^*_3 = \sin(x)$.
\end{enumerate}
In the experiment shown in Fig~\ref{fig:diff_stage}, we choose $m = 5000$, $\alpha = 1.0$ and target function $f^*_1$ to exhibit the training dynamics; in Section~\ref{sec...exp_descent_time}, we use various choices of $m$ and $\alpha$ as stated above to learn $f^*_1$, the result is shown in Fig~\ref{fig:descent_time}; in Section~\ref{sec...exp_target}, we fix $m=5000$,$\alpha=1.0$ and investigate the weights behavior for different target functions $f^*_1$, $f^*_2$, $f^*_3$, the result is shown in Fig~\ref{fig:target_plot}.

\vskip 0.2in

\bibliographystyle{amsplain}
\bibliography{ref}

\end{document}